\documentclass[12pt]{article} 
\usepackage[margin=1in]{geometry} 
\usepackage{graphicx} 
\usepackage{hyperref} 
\usepackage{amsmath} 
\usepackage{graphicx}%
\usepackage{multirow}%
\usepackage{amsmath,amssymb,amsfonts}%
\usepackage{verbatim}
\usepackage{amsthm}%
\usepackage{mathrsfs}%
\usepackage{xcolor}
\usepackage[title]{appendix}%
\newtheorem{theorem}{Theorem}[section]
\newtheorem{corollary}[theorem]{Corollary}
\newtheorem{lemma}[theorem]{Lemma}

\newtheorem{remark}{Remark}[section]%

\newtheorem{definition}{Definition}[section]

\title{Curse of Dimensionality in Neural Network Optimization}
\usepackage{times}
\usepackage{bbm}
\usepackage{mathtools}
\usepackage{esint}

\newcommand{\acknowledgements}[1]{\section*{Acknowledgements}\small{#1}}
\usepackage{authblk}
\usepackage{footnote}

\author[1]{Sanghoon Na\thanks{First author: shna2020@umd.edu}}
\author[2]{Haizhao Yang\thanks{Corresponding author: hzyang@umd.edu}}
\affil[1]{Department of Mathematics, University of Maryland College Park}
\affil[2]{Department of Mathematics and Computer Science, University of Maryland College Park}

\begin{document}

\maketitle

\begin{abstract}
This paper demonstrates that when a shallow neural network with a Lipschitz continuous activation function is trained using either empirical or population risk to approximate a target function that is $r$ times continuously differentiable on $[0,1]^d$, the population risk may not decay at a rate faster than $t^{-\frac{4r}{d-2r}}$, where $t$ denotes the time parameter of the gradient flow dynamics. This result highlights the presence of the curse of dimensionality in the optimization computation required to achieve a desired accuracy. Instead of analyzing parameter evolution directly, the training dynamics are examined through the evolution of the parameter distribution under the 2-Wasserstein gradient flow. Furthermore, it is established that the curse of dimensionality persists when a locally Lipschitz continuous activation function is employed, where the Lipschitz constant in $[-x,x]$ is bounded by $O(x^\delta)$ for any $x \in \mathbb{R}$. In this scenario, the population risk is shown to decay at a rate no faster than $t^{-\frac{(4+2\delta)r}{d-2r}}$. Understanding how function smoothness influences the curse of dimensionality in neural network optimization theory is an important and underexplored direction that this work aims to address.
\end{abstract}


\section{Introduction}
The curse of dimensionality refers to the exponential growth of computational complexity or data requirements with respect to the dimension of the computation or input space. This phenomenon arises in various fields, including Nearest Neighbor algorithms~\cite[Chapter 19]{shalev2014understanding}, numerical methods for solving partial differential equations~\cite{ames2014numerical}, and kernel-based methods~\cite{von2004distance}. It is also observed in the theory of artificial neural networks, particularly in approximation theory~\cite{mhaskar1996neural,devore1989optimal,grohs2023lower,yarotsky2017error,yarotsky2018optimal,CiCP-28-5,SHEN2022101,doi:10.1137/20M134695X} and generalization theory~\cite{li2022robust,yu2024generalizability,JMLR:v25:22-0719}. 

The significance of the curse of dimensionality extends beyond infeasible computational complexity and limited resources; it also restricts a model’s ability to learn and generalize, particularly in high-dimensional spaces. Therefore, understanding this phenomenon and developing strategies to overcome it remains a crucial research topic. In neural network approximation and generalization theory, the theoretical analysis and design of neural network architectures or algorithms to mitigate the curse of dimensionality—whether in terms of the required number of network parameters or the necessary amount of data—are active areas of research ~\cite{bach2017breaking,bach2024learning,poggio2017and,suzuki2018adaptivity,montanelli2019deep,cabannes2021overcoming,grohs2022mathematical,berner2020analysis,JMLR:v25:22-0719,chen2023deep,chen2024letdatatalkdataregularized}.

The curse of dimensionality has rarely been explored in the context of neural network optimization theory, particularly concerning the computational expense of gradient descent-based training. This is largely due to the inherently challenging nature of the non-convex optimization problem. Extensive research has been dedicated to analyzing convergence properties under an over-parameterized (i.e., sufficiently wide) regime~\cite{allen2019convergence,chen2022feature,du2018gradient,du2019gradient,li2020learning,liu2022loss,oymak2020toward,soltanolkotabi2018theoretical,zhou2021local,zou2019improved,zou2020gradient}. Most of these studies aim to establish positive results, demonstrating linear convergence to global or local minima of the empirical risk function with high probability, provided that certain assumptions on network width and training data hold. Although a negative result was shown in~\cite{shamir2019exponential} that exponential convergence time can occur, this result is derived from a one-dimensional linear neural network—a highly specific and atypical setting. Furthermore, in that case, the exponential dependence is only related to the depth of the neural network, instead of the dimension of the learning target.

While the curse of dimensionality in neural network optimization remains an open question, an interesting negative result is presented in~\cite{wojtowytsch2020can} for shallow network training without imposing any assumptions on network width. Specifically, it is shown that there exists a Lipschitz continuous target function for which the population risk cannot decay faster than $t^{-\frac{4}{d-2}}$ under the gradient flow of either empirical risk or population risk in the mean-field regime.  This result suggests that, in general, when learning Lipschitz continuous functions using a shallow neural network, gradient flow training requires at least $\Omega((\frac{1}{\epsilon})^{\frac{d-2}{4}})$ units of time to achieve a population risk smaller than $\epsilon > 0$. The space of Lipschitz continuous functions is vast, making it unsurprising that a particularly challenging target function can be found within this space, leading to the curse of dimensionality in optimization. 

A fundamental question, however, is whether this curse persists when considering a more restricted and structured function space. In this paper, the focus is placed on smooth function spaces for two primary reasons: 1) Smooth functions frequently arise as solutions to partial differential equations (PDEs). It has been conjectured that deep learning-based PDE solvers may circumvent the curse of dimensionality associated with high-dimensional PDEs \cite{doi:10.1073/pnas.1718942115,SIRIGNANO20181339,RAISSI2019686,GU2021110444}. 2) The smoothness of a learning target can introduce additional beneficial structures that could potentially mitigate learning difficulties. Therefore, it is crucial to determine whether smoothness is the key property required to overcome the curse of dimensionality. To address this question, the impact of target function smoothness on the curse of dimensionality in neural network optimization is investigated, a topic that has not been extensively explored in the literature. In particular, the results obtained align with findings in neural network approximation theory, where it is well established that, in general, a shallow neural network requires $O(\epsilon^{-\frac{d}{r}})$ neurons to approximate a function in $C^r$ within a $d$-dimensional space~\cite{mhaskar1996neural,yarotsky2017error,yarotsky2018optimal,yarotsky2020phase}. 

The answer to the fundamental question raised above can be formalized as follows.

\begin{theorem}\label{informal-1}
    Let the training samples be independent and identically distributed from the uniform distribution on $[0,1]^d.$ Let $\sigma:\mathbb{R}\to\mathbb{R}$ be a Lipschitz continuous activation function and $r$ be a positive integer with $r<d/2.$ There exists a target function $\phi\in C^r([0,1]^d)$ such that, when a shallow neural network with activation function $\sigma$ is trained in the mean-field regime by the gradient flow of either the population risk or the empirical risk to learn $\phi$, then
    \[
    \limsup_{t \to \infty} \big[ t^{\gamma} \|f_t-\phi\|_{L^2([0,1]^d)}^2 \big] = \infty,
    \]
    holds for all $\gamma > \frac{4r}{d-2r}.$ Here, $f_t$ denotes the shallow neural network at training time $t.$
\end{theorem}
In the worst-case scenario, the $L^2$ population risk cannot decay at a rate faster than $t^{-\frac{4r}{d-2r}}$.  This result suggests that, in general, when learning a $r$ times continuously differentiable function using a shallow neural network, gradient flow training requires at least $\Omega((\frac{1}{\epsilon})^{\frac{d-2r}{4r}})$ units of time to achieve a population risk smaller than $\epsilon>0$. For fixed $\epsilon$ and $r$, this quantity grows exponentially with the dimension $d$, illustrating the curse of dimensionality in neural network training. Note that there is no assumption on the number of training samples and the neural network width, which makes Theorem~\ref{informal-1} holds uniformly. A more formal mathematical statement for Theorem~\ref{informal-1} appears as Theorem~\ref{global}.

Furthermore, a new problem concerning the impact of activation functions on the curse of dimensionality in neural network optimization is addressed. While most commonly used activation functions, such as the rectified linear unit (ReLU), Gaussian-error linear unit (GELU), Sigmoid, Tanh, Swish, and Sinusoid, are Lipschitz continuous, a growing body of research has focused on activation functions that do not possess this property. Examples include the quadratic activation function $\sigma(x) = x^2$, which has been utilized to analyze the optimization landscape and generalization ability of shallow neural networks~\cite{soltanolkotabi2018theoretical,du2018power,sarao2020optimization}, as well as the $\text{ReLU}^{k}$ activation function (or Rectified Power Unit) $\sigma(x) = \max\{0, x\}^k$, which has been applied in neural network approximation theory~\cite{yang2024optimal,siegel2022high,siegel2024sharp,yang2023nearly,yang2025nearlyoptimalapproximationrates} and in the study of partial differential equations~\cite{luo2020two}. Recently, an advanced activation function—comprising a combination of ReLU, the floor function $[x]$, the exponential function $2^x$, and the Heaviside function $\mathbf{1}_{x\geq 0}$—was proposed in~\cite{shen2021deep,shen2021neural} to enhance the approximation power of neural networks. As the study of novel activation functions continues to advance, it is natural to investigate how these functions influence the curse of dimensionality in neural network optimization. This issue has not been addressed in the literature, e.g., only Lipschitz continuous activation functions are considered in~\cite{wojtowytsch2020can}. In this paper, we settle this question for a broad family of locally Lipschitz continuous activation functions that includes several favorable cases, such as the quadratic activation and the $\text{ReLU}^{k}$ activation. The formal statement appears in Theorem \ref{local}.

Our contributions in this paper can be summarized as follows.
\begin{itemize}
    \item We establish in Theorem~\ref{poor} that in general, $C^r([0,1]^d)$ functions with $r<d/2$ are poorly approximated by two-layer neural networks. Specifically, the optimal approximation rate in the $L^2([0,1]^d)$ topology using Barron functions with a Barron norm bounded by $\kappa$ cannot exceed the rate $\kappa^{-\frac{2r}{d-2r}}$ for such functions. Note that our approximation result differs from the majority of neural network approximation theory literature, which typically expresses approximation rates in terms of the number of parameters in the network architecture. As a corollary, it is proven that $C^r([0,1]^d)$ is not contained in the Barron space when $r < d/2$. Although sufficient regularity, specifically $r > d/2 + 1$, is known to guarantee that a function belongs to the Barron space~\cite{barron1993universal,ma2022barron}, no prior results have explored the relationship between function regularity and Barron spaces for lower regularity. Our findings demonstrate that regularity with $r < d/2$ is insufficient for a function to belong to the Barron space.
    \item In Theorem~\ref{global}, we prove that learning functions that suffer from poor approximation, as identified in Theorem~\ref{poor}, requires an exponential training time under gradient flow to achieve a desired accuracy, leading to the curse of dimensionality in optimization. Mathematically, under gradient flow training of two-layer neural networks in the mean field regime, the population risk cannot decay faster than $t^{-\frac{4r}{d-2r}}$, highlighting the curse of dimensionality in neural network optimization. In Theorem~\ref{global}, we analyze the case of Lipschitz continuous activation functions and allow infinite-width shallow network training. If the activation function is continuously differentiable, then the gradient flow is guaranteed to exist. For piecewise differentiable activation functions such as ReLU and others, the existence of gradient flows has been established in~\cite{chizat2018global,wojtowytsch2020convergence}. Note that our focus is not on the existence of gradient flows.
    \item Finally, in Theorem~\ref{local}, we demonstrate that the curse of dimensionality in neural network optimization persists even when using locally Lipschitz continuous activation functions. Specifically, we show that in general, the population risk under gradient flow training of finite-width shallow neural network in the mean-field regime cannot decay faster than $t^{-\frac{(4+2\delta)r}{d-2r}}$ when learning a $C^r([0,1]^d)$ function using locally Lipschitz continuous activation function whose Lipschitz constant in $[-x,x]$ is bounded by $O(x^\delta)$ for any $x \in \mathbb{R}$. Notably, the activation functions $\sigma(x) = x^2$ and $\sigma(x) = \max\{0, x\}^k$ satisfy this condition.  
\end{itemize}

To the best of our knowledge, our work is the first mathematical paper that demonstrates the influence of the target function's regularity on the curse of dimensionality in neural network training. In contrast to most works in neural network optimization, our theorems do not impose any conditions on the neural network width or the sample size. The only assumption that the data is drawn from the uniform distribution in $[0,1]^d$,  a natural choice that lets us identify the population risk with one-half of the square of the $L^2$ norm.

\section{Related Works}
\subsection{Mean-field theory of neural networks and Wasserstein gradient flows}

The study of Wasserstein gradient flows in neural network training is motivated by the observation that, under mean-field scaling, the evolution of neural network parameters under time-accelerated gradient flow can be equivalently described by the evolution of their distribution via the 2-Wasserstein gradient flow. This framework not only characterizes the training dynamics of two-layer neural networks with a finite number of neurons due to the aforementioned equivalence, but also provides the advantage of describing the training process for infinite-width shallow neural networks. The application of Wasserstein gradient flows in the mean-field regime has been successful in analyzing the convergence of neural network training~\cite{chizat2018global,mei2018mean,rotskoff2018neural,lu2020mean,nitanda2021particle}. For a more detailed discussion on Wasserstein gradient flows in neural network training, see~\cite{chizat2018global,chizat2020implicit,wojtowytsch2020convergence}. For a broader perspective on Wasserstein gradient flows and optimal transport, refer to~\cite{villani2009optimal}.

\subsection{Barron spaces}

Barron spaces, introduced in~\cite{ma2022barron}, generalize Barron's seminal work~\cite{barron1993universal} in neural network approximation theory. A function $f: X \to \mathbb{R}$ is said to be a Barron function if it admits an integral representation of the form  
\begin{align}\label{intrep}
    f(x) = \int_{\mathbb{R} \times \mathbb{R}^{d} \times \mathbb{R}} a \sigma(w^T x + b) \pi(da \otimes dw \otimes db), \quad x \in X,
\end{align}  
for some Borel probability measure $\pi$ on $\mathbb{R} \times \mathbb{R}^{d} \times \mathbb{R}$, and if its Barron norm, as defined in~\cite{ma2022barron}, is finite. Here, $\sigma: \mathbb{R} \to \mathbb{R}$ denotes an activation function. Barron functions can be approximated by finite-width two-layer neural networks with a dimension-independent approximation rate in terms of the number of parameters, given various choices of activation functions. For example,~\cite{ma2022barron} demonstrated that for the ReLU activation, these functions exhibit both dimension-independent approximation rates and low statistical complexity. These favorable properties extend to Barron functions with certain other activation functions~\cite{li2020complexity}.

Both finite-width and infinite-width two-layer networks in the mean-field scaling can be expressed in the integral form given in~\eqref{intrep}. This integral representation facilitates the study of parameter evolution under gradient flow by leveraging the 2-Wasserstein gradient flow in the parameter distribution, as discussed in~\cite{wojtowytsch2020can,wojtowytsch2020convergence}. For a more comprehensive discussion on Barron spaces, see~\cite{ma2022barron,li2020complexity,weinan2022representation,wojtowytsch2021kolmogorov}. 

 In this work, we focus on Barron functions defined by~\eqref{intrep} equipped with the finite Barron norms adopted from~\cite{ma2022barron,li2020complexity,weinan2022representation,wojtowytsch2021kolmogorov}. Crucially, these norms are well-defined for only Lipschitz continuous activation functions, and therefore our primary analysis of Barron spaces is restricted to this class. For non-Lipschitz continuous activations, where these norms may be incompatible with our tools using 2-Wasserstein gradient flows, we focus instead on finite-width networks, which remains a realistic and important setting. We defer the detailed discussion to Section \ref{maintheoremproof} and Appendix \ref{localproof}. For further  literature on Barron spaces for general activation functions, including approximation rates and embedding properties, see~\cite{heeringa2024embeddings,siegel2020approximation,siegel2022high,siegel2024sharp}.

\section{Setup}\label{setup}
The main notations of this paper are listed below.
\begin{itemize}
    \item $Q = [0,1]^d$ denote the unit cube in $\mathbb{R}^d$.
    \item $U_Q$ represent the uniform distribution on $Q$.
    \item The set of continuous functions on $Q$ is denoted by $C(Q)$.
    \item $C^r(Q)$ denotes the set of functions that are $r$ times continuously differentiable on $Q$.
    \item For a normed vector space $X$, the norm of an element $x \in X$ is denoted by $\|x\|_X$.
    \item When $X = \mathbb{R}^d$, this refers to the Euclidean norm, which is simply written as $\|x\|$.
    \item For $x\in\mathbb{R}^d,$ the 1-norm for $x$ is denoted by $|x|_1=\sum_{i=1}^d |x_i|.$
    \item $B^X$ denotes the set of elements in $X$ with norm at most $1$, corresponding to the closed unit ball centered at $0$ in $X$.
    \item The open ball of radius $\epsilon$ centered at $x$ is denoted by $B_{\epsilon}(x)$ and when $x=0,$ we simply write as $B_\epsilon$.
    \item $B'_{\epsilon}(x)$ denotes the projection of the ball $B_{\epsilon}(x)$ from $\mathbb{R}^d / \mathbb{Z}^d$ onto $Q$.
\end{itemize}

    To illustrate this notation, consider the following examples: When $d=1$, the set $B'_{1/8}(1/16)$ is given by $B'_{1/8}(1/16) = [0, 3/16) \cup (15/16,1]$. When $d=2$, the set $B'_{1/4}(\frac{1}{2},0)$ is given by $B'_{1/4} \left( \frac{1}{2}, 0 \right) = \left\{ (x,y) \in [0,1]^2 : \left(x - \frac{1}{2} \right)^2 + y^2 \leq \left(\frac{1}{4} \right)^2 \right\} \cup \left\{ (x,y) \in [0,1]^2 : \left(x - \frac{1}{2} \right)^2 + (y - 1)^2 \leq \left(\frac{1}{4} \right)^2 \right\}$. If a constant $\alpha\in\mathbb{R}^n$ is written as $\alpha=\alpha_{\beta_1,\cdots,\beta_m}$ for some parameters $\beta_1,\cdots,\beta_m$, this means $\alpha$ is a constant that depends only on $\beta_1,\cdots,\beta_m$.

Let $f: X \to \mathbb{R}$ be a function defined on a compact set $X \subset \mathbb{R}^d$. Define $D_f$ as the collection of appropriate Borel probability measures $\pi$ in the integral representation~\eqref{intrep} to generate $f$. The Barron space consists of functions that admit the integral representation~\eqref{intrep} with a finite Barron norm. To emphasize the dependence on the activation function $\sigma$, the Barron space is denoted as $B_{\sigma}(X)$. In this paper, we only consider the Barron space where $\sigma$ is a Lipschitz continuous activation function.  If $D_f$ is nonempty, the Barron norm $\|\cdot\|_{B_\sigma}(X)$ of a Barron function $f$ is defined as follows: When $\sigma$ is the ReLU activation function, the norm is given by 
\begin{align}\label{ReLUnorm}
        \|f\|_{B_\sigma(X)}\coloneqq\inf_{\pi\in D_f}\mathbb{E}_\pi[|a|(\|w\|_1+|b|)] < \infty.
\end{align}
Otherwise, for other Lipschitz continuous activation functions, the norm is defined as
\begin{align}\label{othernorm}
    \|f\|_{B_\sigma(X)}\coloneqq\inf_{\pi\in D_f}\mathbb{E}_\pi[|a|(\|w\|_1+|b|+1)] < \infty.
\end{align}
For simplicity, since the primary focus is on the domain $X = Q$, the notation is further simplified by writing $B_\sigma$ henceforth.

If a Borel probability measure $\pi$ generates a function $f$ through the integral representation ($\ref{intrep}$), we denote it as $f_\pi$ to emphasize its dependence of $\pi.$ In this paper, we assume that the data distribution is uniform on $Q$. Then given a target function $f^*$, the population risk $\mathcal{R}_p$ and the empirical risk $\mathcal{R}_n$ are
\begin{align*}
    \mathcal{R}_p(\pi):=\frac{1}{2}\int_{Q}(f_\pi(x)-f^*(x))^2dx, \quad \mathcal{R}_n (\pi):=\frac{1}{2n}\sum_{i=1}^{n}(f_\pi(x_i)-f^*(x_i))^2
\end{align*}
where $x_i$'s are independent and identically distributed training samples drawn from the uniform distribution on $Q$. Note that the population risk can be written as  $\frac{1}{2}\|f_\pi-f^*\|^2_{L^2([0,1]^d)}$. For a Borel probability measure $\pi,$ we denote its second moment as
\begin{align*}
    N(\pi):=\int_{\mathbb{R}\times\mathbb{R}^d\times\mathbb{R}} a^2+\|w\|^2+b^2 \pi(da \otimes dw \otimes db).
\end{align*}

In this paper, we adopt the framework of \cite{wojtowytsch2020can} and investigate  gradient flow training as evolution of the parameter distribution under the 2-Wasserstein gradient flow. In this framework, if training started with the initial parameter distribution $\pi^0,$ the two-layer neural network at time $t$ is described by $f_{\pi^t}$, where $\pi^t$ is the parameter distribution at time $t$ under the 2-Wasserstein gradient flow. In the finite-width training regime with $m$ neurons, we denote the parameter distribution at time $t$ as $\pi_m^t$ to emphasize the number of neurons.

\section{Main Theorems}
Our first result establishes the existence of functions in $C^r(Q)$ that are poorly approximable by shallow neural networks. This result is formally stated in the following theorem.
\begin{theorem}\label{poor}
    Let $\sigma: \mathbb{R} \to \mathbb{R}$ be a Lipschitz continuous activation function, and let $r$ be a positive integer such that $r < d/2$. Then, there exists a function $\phi \in C^r(Q)$ satisfying
    \[
    \limsup_{\kappa \to \infty} \bigg[ \kappa^{\gamma} \inf_{\|f\|_{B_\sigma} \leq \kappa} \|\phi - f\|_{L^p(Q)} \bigg] = \infty,
    \]
    for any $\gamma > \frac{r/d}{1/2 - r/d} = \frac{2r}{d - 2r}$ and any $p \in [2, \infty]$.
\end{theorem}

 A direct consequence of Theorem~\ref{poor} is the relationship between the smoothness of function spaces and Barron spaces, as stated in the following corollary. This follows from a simple observation: if $C^r(Q)\subset B_\sigma(Q),$ then $\limsup_{\kappa \to \infty} \big[ \kappa^{\gamma} \inf_{\|f\|_{B_\sigma} \leq \kappa} \|\phi - f\|_{L^p(Q)} \big]=0$ for any $\phi \in C^r(Q),$ which contradicts Theorem \ref{poor}. Note that it has been established that when $r>d/2+1,$ $C^r$ functions belong to the Barron space with ReLU activation functions; see \cite[Section 4, point 15]{barron1993universal} and \cite[Corollary B.6]{wojtowytsch2020convergence}.
\begin{corollary}
    Let $\sigma: \mathbb{R} \to \mathbb{R}$ be a Lipschitz continuous activation function, and let $r$ be a positive integer such that $r < d/2$. Then $C^r(Q) \not\subset B_\sigma(Q)$.
\end{corollary}

For functions that suffer from poor approximation as stated in Theorem~\ref{poor}, the curse of dimensionality also manifests in the number of training steps required when they are learned by shallow neural networks. This result is formally stated in the following theorem. Note that within the framework of Section~\ref{setup}, the training dynamics of a shallow neural network can be equivalently interpreted as the evolution of parameter measures $\pi^t$.

\begin{theorem}\label{global}
    Let $\sigma:\mathbb{R}\to\mathbb{R}$ be a Lipschitz continuous activation function and $r$ be a positive integer with $r<d/2.$ There exists a target function $\phi\in C^r(Q)$ with $\|\phi\|_{C^r} \leq 1$ such that, if the parameter measures $\pi^t$ with $N(\pi^0)<\infty$ evolve by the 2-Wasserstein gradient flow of either the population or the empirical risk, then they satisfy  
    \[
    \limsup_{t \to \infty} \big[ t^{\gamma} \mathcal{R}_p(\pi^t) \big] = \infty,
    \]
    for all $\gamma >  \frac{4r}{d - 2r}$. Here, the parameter measures $\pi^t$ define integral representations~\ref{intrep} of shallow neural networks $f_{\pi^t}$ with activation $\sigma$ under the training to learn $\phi.$
\end{theorem}

Theorem~\ref{global} holds uniformly in both the network width and the number of training data. This is due to two key reasons: 1) Theorem~\ref{poor} is an approximation result related to the size of Barron norms, not to the number of parameters. 2) The growth of the second moment of the parameter measure is at most sublinear in time, which is stated in Lemma~\ref{Second moment evolution}, and this bound does not involve the sample size.

For a certain class of locally Lipschitz activation functions, similar results hold when training finite-width shallow neural networks. This result is formally stated in the following theorem.
\begin{theorem}\label{local}
    Let $r$ be a positive integer with $r < d/2$, and let $\sigma$ be a locally Lipschitz continuous activation function. Define $L_x$ as the Lipschitz constant of $\sigma$ on the closed interval $[-x, x]$. Assume that $L_x = O(x^\delta)$ for some $\delta \ge0$. Then, for any positive integer $m$, there exists a function $\phi \in C^r(Q)$ such that if the parameter measures $\pi_m^t$, with $m$ neurons, evolve by the 2-Wasserstein gradient flow of either the population or empirical risk, then they satisfy  
    \[
    \limsup_{t \to \infty} \big[ t^{\gamma} \mathcal{R}_p(\pi_m^t) \big] = \infty,
    \]
    for all $\gamma >  \frac{(4 + 2\delta)r}{d - 2r}$. Here, the parameter measures $\pi^t_m$ define integral representations~\ref{intrep} of shallow neural networks $f_{\pi^t_m} $  with activation $\sigma$ and $m$ neurons, under the training to learn $\phi.$ 
\end{theorem}
Theorem~\ref{local} states that once we fix positive integers $r$ and $m$, then for any dimension $d>2r$, there exists $\phi \in C^r(Q)$ such that $\Omega((\frac{1}{\epsilon})^{\frac{d-2r}{(4+2\delta)r}})$ of units of time may be insufficient to achieve the population risk less than $\epsilon>0$ through gradient flow training via shallow neural network with $m$ neurons. This demonstrates the curse of dimensionality in finite-width shallow neural network training. Although $\phi$ depends on the width $m$, Theorem~\ref{local} holds uniformly in the number of training data. This uniformity follows from Lemma~\ref{Second moment evolution}, which is independent of the sample size.

It is noteworthy that when $\delta = 0$, the activation function $\sigma$ is globally Lipschitz continuous. In this case, Theorem~\ref{local} corresponds to the finite-width shallow neural network training result established in Theorem~\ref{global}.

\section{Key Lemmas}
\subsection{Growth of second moments under the 2-Wasserstein gradient flow}
An important lemma is introduced to demonstrate the sublinear growth of second moments under the 2-Wasserstein gradient flow. Consider the function $f_\pi(x) = \int_{\Theta} \phi(\theta, x) \pi(d\theta)$, expressed as an integral representation of parametrized functions $\{\phi(\theta, x)\}_{\theta \in \Theta}$. Let $f^*$ be the target function to be learned, and define the risk functional as  
\begin{align}\label{risk}
    \mathcal{R}(\pi) = \frac{1}{2} \int_{\mathbb{R}^d} (f_\pi(x) - f^*(x))^2 \mathbb{P}(dx),
\end{align}
for some data distribution $\mathbb{P}$ on $\mathbb{R}^d$. For instance, when $\mathbb{P}$ is the uniform distribution on $Q$, the risk functional \eqref{risk} corresponds to the population risk. Alternatively, if $\mathbb{P}$ is the empirical measure associated with a finite set of training samples $\{x_i\}_{i=1}^{n} \subset Q$, i.e., $\mathbb{P} = \frac{1}{n} \sum_{i=1}^{n} \delta_{x_i}$, then \eqref{risk} represents the empirical risk. In either case, if the parameter distribution $\pi^t$ evolves according to the 2-Wasserstein gradient flow of the risk functional $\mathcal{R}$, then the second moment $N(\pi^t)$ exhibits at most sublinear growth over time. This result is formally stated in the following lemma.

\begin{lemma}[{\cite[Lemma 3.3]{wojtowytsch2020convergence}}]\label{Second moment evolution}
    If $\pi^t$ evolves according to the 2-Wasserstein gradient flow of $\mathcal{R}$ and satisfies $N(\pi^0) < \infty$, then  
    \begin{align}
        N(\pi^t) \leq 2 [N(\pi^0) + \mathcal{R}(\pi^0) t]. 
    \end{align}
\end{lemma}

\subsection{Barron norms and second moments}\label{barron and second moments}
Let $\sigma: \mathbb{R} \to \mathbb{R}$ be an $L$-Lipschitz continuous activation function, and let $X \subset \mathbb{R}^d$ be a compact domain. Then, for any Barron function in $B_\sigma(X)$, bounds on its Barron norm can be established.

\begin{lemma}\label{BarronLipschitz}
    Any function $f \in B_\sigma(X)$ is Lipschitz continuous, with its Lipschitz constant bounded above by $L\|f\|_{B_\sigma(X)}$.
\end{lemma}

The proof is straightforward and is provided in Appendix~\ref{Barronnorm}. This result yields a lower bound on $\|f\|_{B_\sigma(X)}$. The following lemma provides an upper bound using the second moments of $D_f$. This result plays a crucial role in the proofs of Theorems~\ref{global} and~\ref{local}. For its proof, see Appendix~\ref{Barronnorm}.

\begin{lemma}\label{Control of Barron norm}
    Let $\pi$ be a Borel probability measure on $\mathbb{R} \times \mathbb{R}^d \times \mathbb{R}$ such that $N(\pi) < \infty$. Then, the integral representation~\eqref{intrep} defines a Barron function with its Barron norm bounded above by $\|f\|_{B_\sigma(X)} \leq \left(\frac{\sqrt{d}}{2} + 1 \right) N(\pi) + \frac{1}{2}$.
\end{lemma}

\subsection{Slow approximation property across infinitely many time scales}
The following sequence plays a crucial role in constructing an element in a Banach space that exhibits poor approximation properties under appropriate time scales.

\begin{definition}\label{superex}
    A sequence $\{n_k\} \subset \mathbb{N}$ is said to be a \textit{super-exponentially increasing sequence} if it is strictly increasing with $n_1 \geq 2$ and satisfies  
    \begin{align}
        \sum_{l>k} \frac{1}{n_l} \leq \frac{2}{n_k^{k+1}}
    \end{align}
    for all $k \in \mathbb{N}$.
\end{definition}
A super-exponentially increasing sequence $\{n_k\}$ satisfies the inequality  
\begin{align*}
    \sum_{i\ge1} \frac{1}{n_i} = \frac{1}{n_1} + \sum_{l>1} \frac{1}{n_l} \leq \frac{1}{n_1} + \frac{2}{n_1^2} \leq \frac{1}{2} + \frac{2}{2^2} = 1.
\end{align*}
This implies that $\lim_{k\rightarrow\infty} n_k = \infty$. An example of a super-exponentially increasing sequence is given by $n_k = 2^{k^k}$. To verify this, observe that  
\begin{align*}
        \sum_{l> k}\frac{1}{n_l}&\le \sum_{j\ge1}2^{-k^k(k+j)^j}=\sum_{j\ge1} (\frac{1}{n_k})^{-(k+j)^j} \le \sum_{j\ge1} (\frac{1}{n_k^{k+1}})^j=\frac{1}{n_k^{k+1}}\times\frac{1}{1-\frac{1}{n_k^{k+1}}}\le \frac{2}{n_k^{k+1}}
    \end{align*}
Thus, the required condition holds for all $k \in \mathbb{N}$.

A technical lemma is introduced to demonstrate that if a sequence of linear operators exhibits different behavior in a Banach space $Y$ and a sequence of subsets $\{X_k\}_{k \geq 1}$, then there exists an element in the unit ball of $Y$ that is poorly approximated by the elements of $X_k$ under certain time scales. The proof is provided in Appendix~\ref{poor1proof}.

\begin{lemma}\label{poor1}
    Let $Y, Z, W$ be normed linear spaces such that $Y$ is a Banach space with a continuous embedding $Y \hookrightarrow Z$. Suppose there exist linear operators $A_n, A \in L(Z, W)$ satisfying  
    \[
        \|A_n - A\|_{L(Y, W)} \geq c_Y n^{-\beta}, \quad \|A_n - A\|_{L(Z, W)} \leq C_Z,
    \]
    for some $0 < \beta < \alpha$ and positive constants $c_Y, C_Z$ that do not depend on $n$. Moreover, suppose there exist a super-exponentially increasing sequence $\{n_k\}$, a sequence $\{m_k\} \subset \mathbb{N}$, and a sequence of subsets $\{X_k\}_{k\ge1} \subset Z$ such that $m_k = n_k^{[\sqrt{k}]}$ and  
    \begin{align}
        \sup_{x\in X_k} \|(A_{m_k} - A)(x)\|_W \leq \frac{c_Y m_k^{-\beta}}{8 n_k} = \frac{c_Y}{8} m_k^{-\beta - \frac{1}{[\sqrt{k}]}}
    \end{align}
    for all $k \geq 1$. Then, there exists an element $y \in B^Y$ such that for every $\gamma > \frac{\beta}{\alpha - \beta}$,  
    \begin{align*}
        \limsup_{k \to \infty} \bigg[ \bigg( \frac{m_k^{\alpha - \beta}}{n_k} \bigg)^\gamma \inf_{x \in X_k} \|x - y\|_Z \bigg] = \infty.
    \end{align*}
    That is, under the time scales $t_k = \big( \frac{m_k^{\alpha - \beta}}{n_k} \big)^\gamma$, the element $y$ is poorly approximated by the elements of $X_k$.
\end{lemma}

\begin{remark}
    There are multiple choices for the sequence $\{m_k\}_{k \geq 1}$, and the choice $m_k = n_k^{[\sqrt{k}]}$ serves as a straightforward example that simplifies the proof. The construction of $y$ is highly dependent on the sequences $\{n_k\}_{k \geq 1}$ and $\{m_k\}_{k \geq 1}$, implying that the existence of such an element $y$ may not be unique. Furthermore, it is easily verified that $\lim_{k \to \infty} \bigg(\frac{m_k^{\alpha - \beta}}{n_k}\bigg)^\gamma = \infty$.
\end{remark}

Using Lemma~\ref{poor1}, the following result can be established, addressing the case where a sequence of linear operators exhibits opposite behavior between two normed linear spaces. This result provides an improvement over~\cite[Lemma 2.3]{wojtowytsch2021kolmogorov}, which required $0 < \beta < \alpha/2$. The following lemma extends this condition to allow $0 < \beta < \alpha$.

\begin{lemma}\label{poor2}
    Let $X, Y, Z, W$ be normed linear spaces such that $X \subset Z$, and let $Y$ be a Banach space with a continuous embedding $Y \hookrightarrow Z$. Suppose there exist linear operators $A_n, A \in L(Z, W)$ satisfying  
    \[
    \|A_n - A\|_{L(X, W)} \leq C_X n^{-\alpha}, \quad  
    \|A_n - A\|_{L(Y, W)} \geq c_Y n^{-\beta}, \quad  
    \|A_n - A\|_{L(Z, W)} \leq C_Z,
    \]
    for some $0 < \beta < \alpha$ and positive constants $C_X, c_Y, C_Z$ that do not depend on $n$. Then, there exists an element $y \in B^Y$ such that for every $\gamma > \frac{\beta}{\alpha - \beta}$,  
    \[
    \limsup_{\kappa \to \infty} \left( \kappa^{\gamma} \inf_{\|x\|_X \leq \kappa} \|x - y\|_Z \right) = \infty.
    \]
\end{lemma}

\begin{proof}
    Choose any super-exponentially increasing sequence $\{n_k\}_{k\ge1}\subset \mathbb{N}$ and $m_k=n_k^{[\sqrt{k}]}.$ Now let $X_k=t_kB^X$ for $k\ge1,$ where $t_k=\frac{c_Y m_k^{\alpha-\beta}}{8C_X n_k}.$ Then, we have
\begin{align*}
        \sup_{x\in X_k} \|(A_{m_k}-A)(x)\|_W=t_k\|A_{m_k}-A\|_{L(X,W)}\le t_kC_Xm_k^{-\alpha}=\frac{c_Ym_k^{-\beta}}{8n_k}.
\end{align*}
Now from Lemma \ref{poor1}, there exists $y\in B^Y$ such that 
\begin{align*}
        \limsup_{k\rightarrow\infty}\big[ \big(\frac{m_k^{\alpha-\beta}}{n_k}\big)^\gamma \inf_{\|x\|_X\le t_k} \|x-y\|_Z \big]=\infty
    \end{align*}
holds for every $\gamma>\frac{\beta}{\alpha-\beta}$. Since $\frac{m_k^{\alpha-\beta}}{n_k}=\frac{8C_X}{c_Y}t_k,$ this implies
    \begin{align*}
    \limsup_{k\rightarrow\infty}\big(t_k^\gamma \inf_{\|x\|_X\le t_k} \|x-y\|_Z \big)=\infty
    \end{align*}
for every $\gamma>\frac{\beta}{\alpha-\beta}$. Note that since $\alpha>\beta$ and $\lim_{k\rightarrow\infty} n_k=\infty,$ we have
\begin{align*}
        \lim_{k\rightarrow \infty} t_k= \lim_{k\rightarrow \infty}\frac{c_Y m_k^{\alpha-\beta}}{8C_X n_k} = \lim_{k\rightarrow \infty}\frac{c_Y}{8C_X}n_k^{(\alpha-\beta)[\sqrt{k}]-1} = \infty.
\end{align*}
Therefore, we conclude
\begin{align*}
        \limsup_{\kappa \rightarrow \infty} \left( \kappa^{\gamma}\inf_{\|x\|_X\le \kappa} \|x-y\|_{Z} \right) \ge \limsup_{k\rightarrow\infty}\big(t_k^\gamma \inf_{\|x\|_X\le t_k} \|x-y\|_Z \big)=\infty
\end{align*}
for every $\gamma>\frac{\beta}{\alpha-\beta}$.
\end{proof}

\subsection{Low complexity estimates}
To apply Lemmas~\ref{poor1} and~\ref{poor2}, it is necessary to construct appropriate linear operators $\{A_n\}_{n \geq 1}$. Since the focus is on approximation rates in the $L^2$ topology, the space $Z$ in both Lemma~\ref{poor1} and Lemma~\ref{poor2} must be chosen as $L^2(Q)$. In the space $L^2(Q)$, there are functions with undefined point evaluation at certain points. Therefore, following the approach in~\cite{wojtowytsch2021kolmogorov}, we aim to define the linear operators $\{A_n\}_{n \geq 1}$ and $A$ as  
\begin{align}\label{intoperator}
    A_n(f) = \frac{1}{n} \sum_{i=1}^{n} \fint_{B'_{\epsilon_n}(X_n^i)} f(x)dx,  
    \quad A(f) = \int_Q f(x)dx,
\end{align}
where $\{X^i_n\}_{i=1}^{n} \subset Q$ represents an appropriate set of points and $\epsilon_n > 0$ is a suitable radius. The integration over the projected ball $B'_{\epsilon_n}(X_n^i)$ is employed to mitigate boundary effects on the domain $Q$, as also noted in~\cite{wojtowytsch2021kolmogorov}. To establish the validity of this approach, we introduce the following lemma, which demonstrates the existence of suitable points $\{X^i_n\}_{i=1}^{n} \subset Q$.

\begin{lemma}\label{suitablepts}
    Let $\sigma$ be an $L$-Lipschitz continuous activation function. Then for any $n \in \mathbb{N}$ and any constant $0 < \gamma_d \ll 1$, which is independent of $n$, there exist $n$ points $\{X^1_n, \cdots, X^n_n\} \subset Q$ satisfying  
    \begin{align*}
        \sup_{\phi \in B^X} \bigg\{ \frac{1}{n} \sum_{i=1}^{n} \fint_{B'_{\epsilon_n}(X^i_n)} \phi \, dx - \int_{Q} \phi \, dx \bigg\} &\leq 6L\sqrt{\frac{2\log(2d)}{n}}, \\
        \sup_{\phi \in B^Z} \bigg\{ \frac{1}{n} \sum_{i=1}^{n} \fint_{B'_{\epsilon_n}(X^i_n)} \phi \, dx - \int_{Q} \phi \, dx \bigg\} &\leq 3C,
    \end{align*}
    for $X = B_\sigma$, $Z = L^2(Q)$, $\epsilon_n = \gamma_d n^{-1/d}$, and $C=C_{d,\gamma_d}$.
\end{lemma}

\begin{proof}
    This result follows directly from~\cite[Lemma 3.3]{wojtowytsch2021kolmogorov} and~\cite[Lemma A.10]{wojtowytsch2021kolmogorov}. Any $\gamma_d$ which satisfies
    \begin{align*}
        \gamma_d\times\fint_{B_1}|x|dx=\frac{cd}{d+1}\frac{1}{[(d+1)w_d]^\frac{1}{d}}
    \end{align*}
    for some absolute constant $c\in(0,1)$ is an appropriate choice, which is described in the proof of ~\cite[Lemma 3.3]{wojtowytsch2021kolmogorov}. Here, $w_d$ denotes the volume of the unit ball in $\mathbb{R}^d.$
\end{proof}

\subsection{Curse of dimensionality in numerical integration}

The final step in applying Lemmas~\ref{poor1} and~\ref{poor2} is to determine an appropriate choice of $\gamma_d$ such that the linear integral operators~\eqref{intoperator}, constructed using the points from Lemma~\ref{suitablepts}, exhibit different behavior in $C^r(Q)$. Intuitively, for properly scaled values of $\epsilon_n$, the integral operators~\eqref{intoperator} should provide a good numerical approximation. Thus, it is natural to approach this problem using techniques from multivariate numerical integration, particularly in the context of the curse of dimensionality. For a detailed discussion on the curse of dimensionality in multivariate numerical integration, see~\cite{hinrichs2014curse,hinrichs2017product}.

The following lemma demonstrates that the worst-case error in approximating integration using the operators in~\eqref{intoperator} suffers from the curse of dimensionality in $C^r(Q)$. In the proof, a function in $C^r(Q)$ is constructed to vanish in every $B'_{\epsilon_n}(X^i_n)$. This approach is inspired by~\cite{hinrichs2014curse,hinrichs2017product}, where a \textit{fooling function} is obtained by applying a sequence of convolutions between a Lipschitz continuous function and scaled indicator functions of a ball. After constructing this function, an appropriate choice of $\epsilon_n$ is determined to control the $C^r$ norm and ensure proper integration over $Q$. The lemma is stated below, with its proof provided in Appendix~\ref{proof-numcod}.

\begin{lemma}\label{numerical_cod}
    Let $C^r(Q)$ denote the space of all $r$-times continuously differentiable functions on $Q$, equipped with the norm  
    \[
    \|f\|_{C^r} := \max_{|\beta|_1 \leq r} \|D^{\beta} f\|_{\infty},
    \]
    where $D^\beta$ represents the partial derivative of order $\beta \in \mathbb{N}_0^d$. Then, there exists a positive constant $\tau=\tau_{r,d}$ such that for any $\epsilon_n = \theta n^{-1/d}$ with $\theta=\theta_{d,r} \in (0, \tau]$ and any $n$ points $\{x_1, \cdots, x_n\} \subset Q$, one can construct a function $\psi \in C^r(Q)$ satisfying  
    \begin{align*}
        \|\psi\|_{C^r} \leq 1, \quad \int_Q \psi(x)dx \geq K_{\theta,d,r} n^{-r/d}, \quad \psi|_{B'_{\epsilon_n}(x_i)} = 0,
    \end{align*}
    for some positive constant $K_{\theta,d,r}$. Consequently, this implies that  
    \begin{align*}
        \sup_{\|g\|_{C^r} \leq 1} \bigg\{ \frac{1}{n} \sum_{i=1}^{n} \fint_{B'_{\epsilon_n}(x_i)} g \,dx - \int_{Q} g \,dx \bigg\} \geq K_{\theta,d,r} n^{-r/d}.
    \end{align*}
\end{lemma}

\section{Proofs of the Main Theorems}\label{maintheoremproof}
In this section, we present the proofs of Theorems~\ref{poor} and~\ref{global}. While these proofs rely on Lemma~\ref{poor2}, the proof of Theorem~\ref{local} instead utilizes Lemma~\ref{poor1}. This distinction arises for two reasons: 1) When $\sigma$ is not a Lipschitz continuous function, the Contraction Lemma~\cite[Lemma 26.9]{shalev2014understanding} cannot be applied to bound the Rademacher complexity of the unit ball in $B_\sigma$. 2) If the activation is not Lipschitz continuous, bounding the Barron norm using the second moments (as in Lemma \ref{Control of Barron norm}) may be infeasible. We provide a more detailed discussion in Appendix~\ref{localproof}. However, if the analysis is restricted to finite-width training, which represents the realistic scenario, arguments similar to those presented in this section can be employed to establish Theorem~\ref{local}. The proof of Theorem~\ref{local} is provided in Appendix~\ref{localproof}.

\subsection{Proof of Theorem \ref{poor}}
\begin{proof}
    Define $X=B_\sigma(Q), Y=C^r(Q)$ equipped with the norm defined in Lemma \ref{numerical_cod}, $Z=L^2(Q),$ and $W=\mathbb{R}.$ Since $Q$ is a compact set, It is clear that $X,Y \hookrightarrow Z. $ Now we find linear operators $A_n,A$ that satisfies conditions in Lemma \ref{poor2}. First, choose $\gamma=\gamma_{d,r}$ as $0<\gamma \ll1$ and $\gamma\le\tau$, where $\tau$ is the constant in Lemma \ref{numerical_cod}. Set $\epsilon_n=\gamma n^{-1/d}.$ For any $n\in\mathbb{N}$, there exists $n$ points $\{X^1_n,\cdots,X^n_n\} \subset Q$ that satisfies Lemma \ref{suitablepts}. Then with Lemma \ref{numerical_cod}, we can conclude that these points satisfy inequalities:
    \begin{align*}
        &\sup_{\phi \in B^X} \big\{\frac{1}{n} \sum_{i=1}^{n} \fint_{B'_{\epsilon_n}(X^i_n)} \phi dx-\int_{Q} \phi dx \big\} \le  6L\sqrt{\frac{2\log(2d)}{n}},\\
        &\sup_{\phi \in B^Y} \big\{\frac{1}{n} \sum_{i=1}^{n} \fint_{B'_{\epsilon_n}(X^i_n)} \phi dx-\int_{Q} \phi dx \big\} \ge K_{\gamma,d,r} n^{-r/d}, \\
        &\sup_{\phi \in B^Z} \big\{\frac{1}{n} \sum_{i=1}^{n} \fint_{B'_{\epsilon_n}(X^i_n)} \phi dx-\int_{Q} \phi dx \big\} \le 3C_{d,\gamma}.    
    \end{align*}
    Define linear operators $A_n,A:Z\rightarrow \mathbb{R}$ as
    \begin{align}\label{operators}
        A_n(\phi)=\frac{1}{n} \sum_{i=1}^{n} \fint_{B'_{\epsilon_n}(X^i_n)} \phi dx, \quad A(\phi)=\int_{Q} \phi dx
    \end{align}
    Then $A_n,A$ satisfy the conditions in Lemma \ref{poor2}. Hence by Lemma \ref{poor2}, there exist a function $f \in C^r(Q)$ with $\|f\|_{C^r}\le 1$ such that
    \begin{align*}
        \limsup_{\kappa\rightarrow\infty}\big(\kappa^\gamma \inf_{\|\phi\|_{B_\sigma}\le \kappa} \|\phi-f\|_{L^2(Q)}\big)=\infty
    \end{align*}
    holds for every $\gamma > \frac{r/d}{1/2-r/d}=\frac{2r}{d-2r}.$ Note that since $Q$ is compact with Lebesgue measure $1$,  $\|g\|_{L^2(Q)} \le \|g\|_{L^p(Q)} $ holds for all continuous function $g$ on $Q$ and for all $p\in[2,\infty].$ Therefore, we can conclude
    \begin{align*}
        \limsup_{\kappa\rightarrow\infty}\big(\kappa^\gamma \inf_{\|\phi\|_{B_\sigma}\le \kappa} \|\phi-f\|_{L^p(Q)}\big)=\infty
    \end{align*}
    holds for every $\gamma > \frac{r/d}{1/2-r/d}=\frac{2r}{d-2r}$ and $p\in[2,\infty]$.
\end{proof}

\subsection{Proof of Theorem \ref{global}}

\begin{figure}[ht!]
    \centering
    \includegraphics[width=0.8\textwidth]{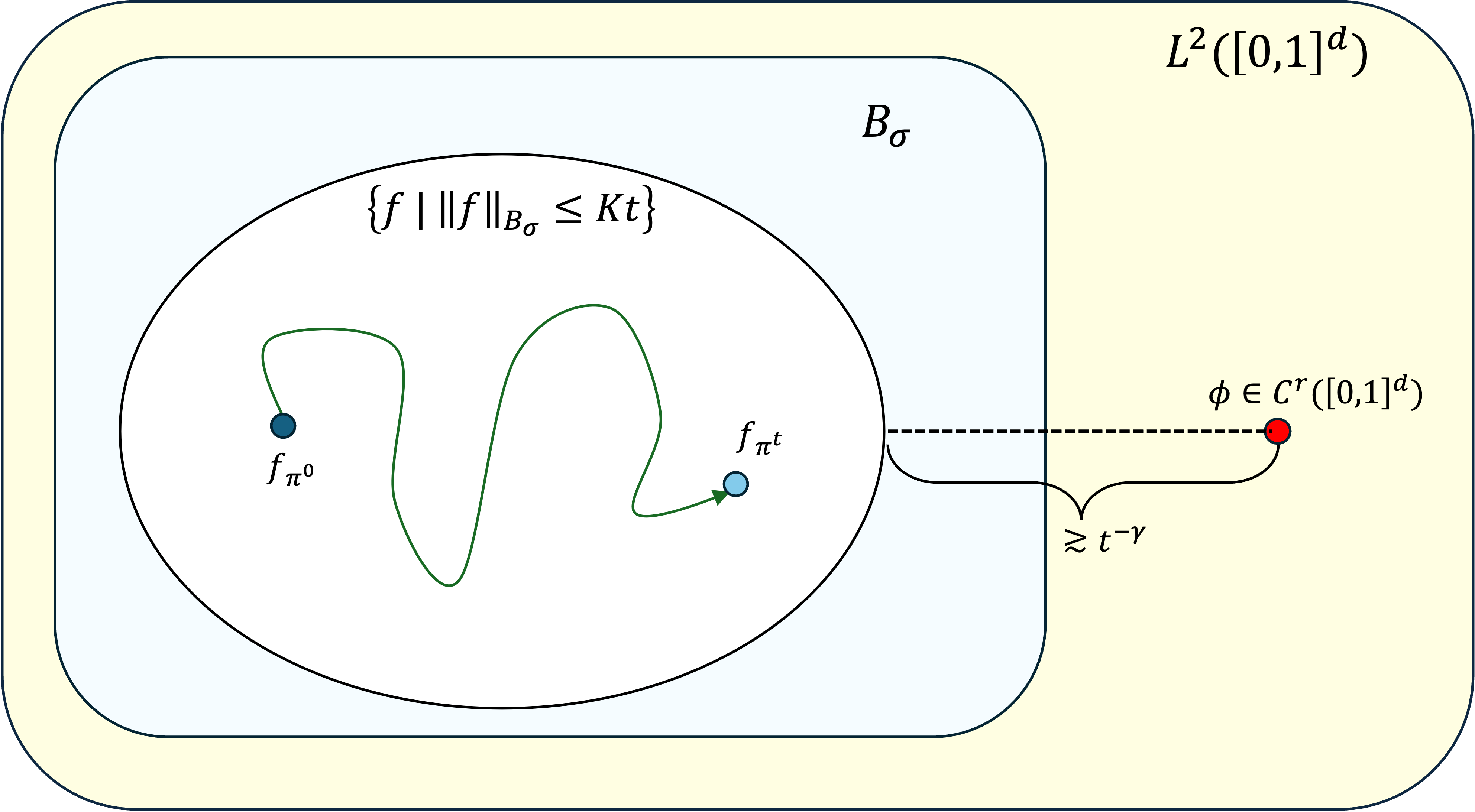}
    \caption{Geometric description of the proof of Theorem~\ref{global}. The green curve with arrow illustrates the sublinear growth of the Barron norm, which follows from Lemma~\ref{Second moment evolution} and Lemma~\ref{Control of Barron norm}. The shallow neural network at the initialization is denoted as $f_{\pi^t}$, represented as a circle filled with dark blue. The shallow neural network training time $t$ is denoted as $f_{\pi^t}$, represented as a circle filled with sky-blue. The black dotted line represents Theorem~\ref{poor}, the existence of $C^r(Q)$ function with slow approximation property.}
    \label{fig 1}
\end{figure}

\begin{proof}
    Choose any $\phi\in C^r(Q)$ that satisfies Theorem \ref{poor}. Let $\pi^0$ be the Borel probability measure at the training initialization with $N(\pi^0)<\infty,$ and let $\pi^t$ be the evolution of $\pi^0$ by the Wasserstein gradient flow at time $t>0$. By Lemma \ref{Second moment evolution}, $N(\pi^t)<\infty.$ and therefore $f_{\pi^t}\in B_\sigma$ holds from Lemma \ref{Control of Barron norm}. Moreover, Lemma \ref{Control of Barron norm} and Lemma \ref{Second moment evolution} give us estimation of the Barron norm of $f_{\pi^t}$ as
    \begin{align*}
        \|f_{\pi^t}\|_{B_\sigma} \le (\frac{\sqrt{d}}{2}+1)N(\pi^t)+\frac{1}{2}\le(\sqrt{d}+2)(N(\pi^0)+R(\pi^0)t)+\frac{1}{2}.
    \end{align*}
    Hence, there exists a positive constant $K=K_{\pi^0,\phi,d}$  such that $\|f_{\pi^t}\|_{B_\sigma}\le Kt$ holds for $t\ge1$. Note that the population risk at time $t$, $\mathcal{R}_p(\pi^t)$, is equal to $\frac{1}{2}\|\phi-f_{\pi^t}\|_{L^2(Q)}^2.$ Now by Theorem \ref{poor},
    \allowdisplaybreaks
    \begin{align*}
        \limsup_{t\rightarrow\infty} \big[ t^{\gamma} \mathcal{R}_p(\pi^t) \big]&=\frac{1}{2}\limsup_{t\rightarrow\infty}\big(t^\gamma \|\phi-f_{\pi^t}\|_{L^2(Q)}^2\big)\ge \frac{1}{2}\limsup_{t\rightarrow\infty}\big(t^\gamma \inf_{\|f\|_{B_\sigma}\le Kt} \|\phi-f\|_{L^2(Q)}^2\big)\\
        &=\frac{1}{2}(\frac{1}{K})^{\gamma}\times\limsup_{t\rightarrow\infty}\big(t^\gamma \inf_{\|f\|_{B_\sigma}\le t} \|\phi-f\|_{L^2(Q)}^2\big)=\infty
    \end{align*} holds for any $\gamma>2\times\frac{r/d}{1/2-r/d}=\frac{4r}{d-2r}.$
\end{proof}

\section{Conclusion}
In this paper, we investigate the curse of dimensionality in the optimization of shallow neural networks. Utilizing theories from Wasserstein gradient flows, Barron spaces, and multivariate numerical integration, we demonstrate that the population risk can decrease at an extremely slow rate, potentially requiring exponential training time to achieve a small error, even for smooth functions. Furthermore, we establish that the curse of dimensionality persists when the activation function is locally Lipschitz continuous. As a supplementary result, we show that a function with smoothness $r < d/2$ cannot be guaranteed to belong to the Barron space.

While our result is the first to analyze the impact of target function's regularity on the curse of dimensionality in neural network training, there are several open questions remaining. Here, we list some of them.

\textbf{Explicit construction :} Theorem~\ref{global} and Theorem~\ref{local} present the existence of $C^r(Q)$ functions suffering from the curse of dimensionality in training, but the proofs rely on probabilistic argument. Therefore, it would be interesting to exhibit an explicit examples of such functions and to characterize them structurally.

\textbf{Loss function :} Our analysis considers training with the $L^2$ loss function to learn target functions with specific regularity. On the other hand, classification tasks typically employ the cross-entropy function for neural network training to learn target functions that differ in nature from those in our analysis. It is therefore worth investigating whether the curse of dimensionality in neural network optimization-the requirement for exponential training time under gradient flow-occurs in this setting. The next question is designing a loss function that encodes a priori information about the target function-for example, physical constraints coming from a PDE. Such an information-rich loss function could potentially help avoid the curse of dimensionality in training.

\textbf{Accelerated gradient descent :} In this paper, we focus on the Wasserstein gradient flow, which captures gradient descent training and stochastic gradient descent training with small step sizes. Exploring whether the curse persists—or can be mitigated—when accelerated methods (e.g., Nesterov or heavy-ball dynamics) are employed is an appealing direction. Developing new acceleration methods to circumvent the curse of dimensionality is also a valuable research direction.

\acknowledgements{The authors were partially supported by the US National Science Foundation under awards DMS-2244988, IIS-2520978, GEO/RISE-5239902, the Office of Naval Research Award N00014- 23-1-2007, DOE (ASCR) Award DE-SC0026052, and the DARPA D24AP00325-00. Approved for public release; distribution is unlimited.}    

\bibliography{references.bib}
\bibliographystyle{plain}

\appendix


\section{Proofs of Subsection \ref{barron and second moments}}\label{Barronnorm}
\subsection{Proof of Lemma \ref{BarronLipschitz}}
\begin{proof}Choose a Borel probability measure $\pi\in D_f$. Then for any $x,y \in X,$ we have
\allowdisplaybreaks
    \begin{align*}
        |f(x)-f(y)|&=\left|  \int_{\mathbb{R}\times\mathbb{R}^{d}\times\mathbb{R}}a(\sigma(w^Tx+b)-\sigma(w^Ty+b))\pi(da\otimes dw\otimes db) \right| \\
        &\le \int_{\mathbb{R}\times\mathbb{R}^{d}\times\mathbb{R}}|a||\sigma(w^Tx+b)-\sigma(w^Ty+b)|\pi(da\otimes dw\otimes db) \\
        &\le \int_{\mathbb{R}\times\mathbb{R}^{d}\times\mathbb{R}}|a|\times L|(w^Tx+b)-(w^Ty+b)|\pi(da\otimes dw\otimes db) \\
        &\le L\int_{\mathbb{R}\times\mathbb{R}^{d}\times\mathbb{R}}|a|\times \|w\|\times\|x-y\|\pi(da\otimes dw\otimes db) \\
        &\le L\|x-y\| \int_{\mathbb{R}\times\mathbb{R}^{d}\times\mathbb{R}}|a|\times \|w\|_1\pi(da\otimes dw\otimes db)\\
        &\le L\|x-y\| \int_{\mathbb{R}\times\mathbb{R}^{d}\times\mathbb{R}}|a|(\|w\|_1+|b|)\pi(da\otimes dw\otimes db)\\
        &=L\times\mathbb{E}_\pi [|a|(\|w\|_1+|b|)]\times\|x-y\|.
    \end{align*}
    Now take infimum over all the set $D_f,$ and we can conclude that $f$ is Lipschitz continuous on $X$ with its Lipschitz constant bounded above by $L\|f\|_{B_\sigma(X)}.$
\end{proof}
\subsection{Proof of Lemma \ref{Control of Barron norm}}
\begin{proof}From Cauchy-Schwarz inequality, $\|w\|_1\le\sqrt{d}\|w\|_2$ holds. Since $\pi$ is a probability measure, $\mathbb{E}_\pi[1]=1$ holds. Therefore we get
\allowdisplaybreaks
    \begin{align*}
        \mathbb{E}_\pi[|a|(\|w\|_1+|b|+1)]&\le \sqrt{d}\times\mathbb{E}_\pi[|a|\|w\|_2]+\mathbb{E}_\pi[|a||b|]+\mathbb{E}_\pi[|a|]\\
        &\le \sqrt{d}\times\mathbb{E}_\pi[\frac{1}{2}a^2+\frac{1}{2}\|w\|_2^2]+\mathbb{E}_\pi[\frac{1}{2}a^2+\frac{1}{2}b^2]+\mathbb{E}_\pi[\frac{1}{2}a^2+\frac{1}{2}] \\
        &\le(\frac{\sqrt{d}}{2}+1)\times\mathbb{E}_\pi[a^2+\|w\|_2^2+b^2]+\frac{1}{2}=(\frac{\sqrt{d}}{2}+1)N(\pi)+\frac{1}{2}. 
    \end{align*}
    Write $K=\max_{x\in X}\|x\|. $ Note that $|\sigma(w^Tx+b)|\le|\sigma(w^Tx+b)-\sigma(0)|+|\sigma(0)|\le L|w^Tx+b|+|\sigma(0)|$ holds due to Lipchitz continuity of $\sigma.$ Then we get
    \allowdisplaybreaks
    \begin{align*}
    &\int_{\mathbb{R}\times\mathbb{R}^{d}\times\mathbb{R}}|a\sigma(w^Tx+b)|\pi(da\otimes dw\otimes db)\\
    &\le \int_{\mathbb{R}\times\mathbb{R}^{d}\times\mathbb{R}}|a|(L|w^Tx+b|+|\sigma(0)|)\pi(da\otimes dw\otimes db)\\
    &\le \max\{L,|\sigma(0)|\}\int_{\mathbb{R}\times\mathbb{R}^{d}\times\mathbb{R}}|a|(|w^Tx|+|b|+1)\pi(da\otimes dw\otimes db)\\
    &\le \max\{L,|\sigma(0)|\}\times\int_{\mathbb{R}\times\mathbb{R}^{d}\times\mathbb{R}}|a|(\|w\|\|x\|+|b|+1)\pi(da\otimes dw\otimes db) \\
    &\le \max\{L,|\sigma(0)|\}\times\max\{K,1\}\times\int_{\mathbb{R}\times\mathbb{R}^{d}\times\mathbb{R}}|a|(\|w\|+|b|+1)\pi(da\otimes dw\otimes db) \\
    &\le \max\{L,|\sigma(0)|\}\times\max\{K,1\}\times\int_{\mathbb{R}\times\mathbb{R}^{d}\times\mathbb{R}}|a|(\|w\|_1+|b|+1)\pi(da\otimes dw\otimes db) \\
    &=\max\{L,|\sigma(0)|\}\times\max\{K,1\}\times \mathbb{E}_\pi[|a|(\|w\|_1+|b|+1)] \\
    &\le \max\{L,|\sigma(0)|\}\times\max\{K,1\}\times\{(\frac{\sqrt{d}}{2}+1)N(\pi)+\frac{1}{2}\} < \infty.
    \end{align*}    
    Hence, integral representation $f(x)=\int_{\mathbb{R}\times\mathbb{R}^{d}\times\mathbb{R}}a\sigma(w^Tx+b)\pi(da\otimes dw\otimes db)$ is well defined for all $x\in X.$ Now the definition of Barron norm gives us
    \begin{align*}
        \|f\|_{B_\sigma(X)}\le \mathbb{E}_\pi[|a|(\|w\|_1+|b|+1]\le(\frac{\sqrt{d}}{2}+1)N(\pi)+\frac{1}{2}.
    \end{align*}
\end{proof}

\section{Proof of Lemma \ref{poor1}}\label{poor1proof}
\begin{proof}We construct such $y$ in the following way.
    
    Since $\|A_n-A\|_{L(Y,W)}=\sup_{\|x\|_Y=1} \|(A_n-A)(x)\|_W \ge c_Y n^{-\beta},$ there exists a sequence $(y_n)_{n\ge 1}$ such that $\|y_n\|_Y=1$ and $\|(A_n-A)(y_n)\|_W \ge \frac{c_Y}{2}n^{-\beta} $ for all $n\ge 1.$ By the Hahn-Banach theorem, there exists a sequence $(w_n^{*})_{n\ge1} \subset W^{*}$ such that $\|w_n^{*}\|_{W^{*}}=1$ and $w_n^{*}\circ(A_n-A)(y_n)=\|(A_n-A)(y_n)\|_W \ge \frac{c_Y}{2}n^{-\beta}$ for all $n\ge1.$
    
    Define a sequence $(\epsilon_k)_{k\ge1}$ where $\epsilon_1=1$ and each $\epsilon_k\in \{-1,1\}$ is chosen inductively such that
    $$ \epsilon_k\times w_{m_k}^{*}\circ(A_{m_k}-A)(\sum_{i=1}^{k-1}\frac{\epsilon_i}{n_i}y_{m_i}) \ge 0,$$
    for all $k\ge2.$ One can easily check that $\big(\sum_{i=1}^{k} \frac{\epsilon_i}{n_i}y_{m_i} \big)_{k\ge1}$ form a Cauchy sequence, and since $Y$ is a Banach space, the infinite sum $y\coloneqq \sum_{i=1}^{\infty} \frac{\epsilon_i}{n_i}y_{m_i} \in Y$ is well defined.

    To shorten notation, define $L_k= w_{m_k}^{*}\circ(A_{m_k}-A).$ Using $Y \hookrightarrow Z, \|A_n-A\|_{L(Z,W)} \le C_{Z},$ and $\|w_n^{*}\|_{W^*}=1,$ one can easily verify that $L_k \in Y^{*}$ for all $k\ge1.$ 
    
    When $\epsilon_k=1,$ we have
    \allowdisplaybreaks
    \begin{align*}
        L_ky&=L_k(\sum_{i=1}^{k-1}\frac{\epsilon_i}{n_i}y_{m_i})+L_k(\frac{\epsilon_k}{n_k}y_{m_k})+L_k(\sum_{l>k}\frac{\epsilon_l}{n_l}y_{m_l}) \\
        &\ge 0+\frac{1}{n_k}L_k(y_{m_k})-C^{Y}\sum_{l>k+1}\frac{1}{n_l} \\
        &= \frac{1}{n_k}\times w_{m_k}^{*}\circ(A_{m_k}-A)(y_{m_k})-C^{Y}\sum_{l>k}\frac{1}{n_l} \\
        &= \frac{1}{n_k}\|(A_{m_k}-A)(y_{m_k})\|_W-C^{Y}\sum_{l>k1}\frac{1}{n_l} \\
        &\ge \frac{1}{n_k}\big(\frac{c_Y}{2}m_k^{-\beta}-C^Yn_k\sum_{l>k}\frac{1}{n_l}\big).
    \end{align*}

    Similarly, when $\epsilon_k=-1,$ we have
    \allowdisplaybreaks
    \begin{align*}
        L_ky&=L_k(\sum_{i=1}^{k-1}\frac{\epsilon_i}{n_i}y_{m_i})+L_k(\frac{\epsilon_k}{n_k}y_{m_k})+L_k(\sum_{l>k}\frac{\epsilon_l}{n_l}y_{m_l}) \\
        &\le 0-\frac{1}{n_k}L_k(y_{m_k})+C^Y\sum_{l>k}\frac{1}{n_l} \\
        &\le -\frac{1}{n_k}\big(\frac{c_Y}{2}m_k^{-\beta}-C^Yn_k\sum_{l>k}\frac{1}{n_l}\big).
    \end{align*}
    Therefore, we have $\frac{1}{n_k}\big(\frac{c_Y}{2}m_k^{-\beta}-C^Yn_k\sum_{l>k}\frac{1}{n_l}\big) \le |L_ky|$ for all $k\ge1.$

    Choose $x_k\in X_k$ as
    $$ \|y-x_k\|_Z \le \inf_{x\in X_k} \|y-x\|_Z+\frac{c_Ym_k^{-\beta{}}}{8C_Zn_k}.$$
    Then,
    \allowdisplaybreaks
    \begin{align*}
        \frac{1}{n_k} \big(\frac{c_Y}{2}m_k^{-\beta}-C^Yn_k\sum_{l=k+1}^{\infty}\frac{1}{n_l}\big) &\le |L_k(y)| \le |L_k(y-x_k)|+|L_k(x_k)| \\
        &\le C_Z\|y-x_k\|_{Z}+\|(A_{m_k}-A)(x_k)\|_{W} \\
        &\le C_Z\big(\inf_{x\in X_{k}} \|y-x\|_{Z}+\frac{c_Ym_k^{-\beta{}}}{8C_Zn_k} \big)+\frac{c_Y m_k^{-\beta}}{8n_k}\\
        &= C_Z\inf_{x\in X_{k}} \|y-x\|_{Z}+\frac{c_Y m_k^{-\beta}}{4n_k},
    \end{align*}
    therefore we get
    \begin{align*}
        \frac{1}{C_Z n_k}\big( \frac{c_Y}{4}m_{k}^{-\beta}- C^Yn_k\sum_{l=k+1}^{\infty}\frac{1}{n_l}  \big) \le \inf_{x\in X_{k}} \|y-x\|_{Z}.
    \end{align*}
Since $\{n_k\}_{k\ge1}$ is a super-exponentially increasing sequence, $m_k=n_k^{[\sqrt{k}]}$ implies $\lim_{k\rightarrow\infty}\frac{n_k^k}{m_k^{\beta}}=\infty$. Hence, there exists $K_0>0$ such that for any $k\ge K_0$ we have
    \begin{align*}
        C^Yn_k\sum_{l=k+1}^{\infty}\frac{1}{n_l} \le C^Yn_k\frac{2}{n_k^{k+1}}=\frac{2C^Y}{n_k^k}\le \frac{c_Y}{8n_k^{\beta[\sqrt{k}]}}=\frac{c_Y}{8}m_k^{-\beta}.
    \end{align*}
    Then for $k\ge K_0,$
    \begin{align*}
        \frac{c_Y}{8C_Z}n_k^{-1-\beta[\sqrt{k}]}=\frac{c_Y m_k^{-\beta}}{8C_Z n_k}\le\frac{1}{C_Z n_k}\big( \frac{c_Y}{4}m_{k}^{-\beta}- C^Yn_k\sum_{l=k+1}^{\infty}\frac{1}{n_l}  \big) \le \inf_{x\in X_{k}} \|y-x\|_{Z}.
    \end{align*}
    Now if we multiply $\big(\frac{m_k^{\alpha-\beta}}{n_k}\big)^\gamma$ on both sides, we get
    \begin{align*}
        \big(\frac{m_k^{\alpha-\beta}}{n_k}\big)^\gamma\times\frac{c_Y}{8C_Z}n_k^{-1-\beta[\sqrt{k}]}=\frac{c_Y}{8C_Z}n_k^{[\sqrt{k}]\times((\alpha-\beta)\gamma-\beta)-\gamma-1} \le \big(\frac{m_k^{\alpha-\beta}}{n_k}\big)^\gamma \inf_{x\in X_{t_k}} \|y-x\|_{Z}.
    \end{align*}
    Note that $\lim_{k\rightarrow\infty}n_k=\infty$ and $\lim_{k\rightarrow\infty} \big\{[\sqrt{k}]\times((\alpha-\beta)\gamma-\beta)-\gamma-1\big\}=\infty$ if $\gamma>\frac{\beta}{\alpha-\beta}.$ Therefore for every $\gamma > \frac{\beta}{\alpha-\beta},$ we can conclude
    \begin{align*}
        \limsup_{k\rightarrow\infty}\big[ \big(\frac{m_k^{\alpha-\beta}}{n_k}\big)^\gamma \inf_{x\in X_k} \|x-y\|_Z \big] \ge \frac{c_Y}{8C_Z}\limsup_{k\rightarrow\infty}n_k^{[\sqrt{k}]\times((\alpha-\beta)\gamma-\beta)-\gamma-1} 
       =\infty.
    \end{align*}
\end{proof}
\section{Proof of Lemma \ref{numerical_cod}}\label{proof-numcod}
We begin with a lemma stating that a convolution of a Lipschitz function and an indicator function of a ball is a continuously differentiable function.
\begin{lemma}\label{convolution}
    Let $f:\mathbb{R^d}\to\mathbb{R}$ be a $K$-Lipschitz continuous function. Let $\mathbbm{1}_{B_r}$ be the indicator function of the ball $B_r$ with center 0 and radius $r>0.$ Then, $f\ast\mathbbm{1}_{B_r}$ is $C^1$ function with its derivatives bounded by $K\times |B_r|,$ where $|B_r|$ is the volume of $B_r$.
\end{lemma}
\begin{proof}
     Fix an index $i \in \{1,\cdots,d\},$ and let $e_i$ be the unit vector in $\mathbb{R}^d$ where all the entries are $0$ except the $i$th coordinate, which is $1$. By Rademacher's Theorem, $f$ is differentiable almost everywhere. Write $\partial_if$ as the derivative of $f$ with respect to $x_i.$ From the definition of the Lipschitz constant and the derivative, it is obvious that $\|\partial_if\|_{\infty}\le K.$ Therefore for any fixed $x\in\mathbb{R}^d$, we have
     \begin{align*}
         &\lim_{h\rightarrow0}\frac{f(x+he_i-y)-f(x-y)}{h}=\partial_if(x-y),\\
         &\Big|\frac{f(x+he_i-y)-f(x-y)}{h}-\partial_if(x-y)\Big|\le 2K
     \end{align*}
     almost everywhere in $y\in\mathbb{R}^d$. Hence, by Lebesgue's Dominated Convergence Theorem, for any $x\in\mathbb{R}^d$ we get
     \allowdisplaybreaks
     \begin{align*}
         &\lim_{h\to0}\frac{f\ast\mathbbm{1}_{B_r}(x+he_i)-f\ast\mathbbm{1}_{B_r}(x)}{h}-\partial_if\ast\mathbbm{1}_{B_r}(x)\\=&\lim_{h\to 0}\int_{B_r}\frac{f(x+he_i-y)-f(x-y)}{h}-\partial_if(x-y) dy \\
         =&\int_{B_r} \Big(\lim_{h\to0} \frac{f(x+he_i-y)-f(x-y)}{h}-\partial_if(x-y)\Big)dy= 0.
     \end{align*}
     Hence, $f\ast\mathbbm{1}_{B_r}$ is differentiable and its partial derivatives satisfy
     \begin{equation*}
         \partial_i (f\ast\mathbbm{1}_{B_r})=\partial_if\ast\mathbbm{1}_{B_r}.
     \end{equation*}
     Since $\|\partial_if\|_\infty \le K,$ it is easy to check $|\partial_i f\ast\mathbbm{1}_{B_r} (x)| \le K\times|B_r|.$ Now, it remains to check that $\partial_if\ast\mathbbm{1}_{B_r}$ is continuous. From the definition of convolution, 
     \begin{equation*}
         \partial_if\ast\mathbbm{1}_{B_r}(x+z)-\partial_if\ast\mathbbm{1}_{B_r}(x)=\int_{\mathbb{R}^d} (\mathbbm{1}_{B_r}(x+z-y)-\mathbbm{1}_{B_r}(x-y))\partial_if(y)dy.
     \end{equation*}
     Note that if $\|z\|<<1$, then $$|\mathbbm{1}_{B_r}(x+z-y)-\mathbbm{1}_{B_r}(x-y)| \le \mathbbm{1}_{B_r(x+z)}(y)+\mathbbm{1}_{B_r(x)}(y)\le 2\times \mathbbm{1}_{B_{r+1}(x)}(y)$$holds. Since $\|\partial_i f\|_\infty \le K$, $|(\mathbbm{1}_{B_r}(x+z-y)-\mathbbm{1}_{B_r}(x-y))\partial_if(y)|$ is bounded by $2K\times \mathbbm{1}_{B_{r+1}(x)}(y)$, which is an integrable function in $\mathbb{R}^d.$ Also note that 
     \begin{align*}
         \lim_{\|z\|\to 0} \mathbbm{1}_{B_r}(x+z-y)-\mathbbm{1}_{B_r}(x-y) =0.
     \end{align*}
     Hence, by Lebesgue's Dominated Convergence Theorem,
     \begin{equation*}
         \lim_{\|z\|\to0}\partial_if\ast\mathbbm{1}_{B_r}(x+z)-\partial_if\ast\mathbbm{1}_{B_r}(x) = 0.
     \end{equation*}
     Therefore, the derivatives of $f\ast \mathbbm{1}_{B_r}$ are continuous.
\end{proof}

Now we present the proof of Lemma \ref{numerical_cod}. 

\begin{proof}
Let $\{x_1,\cdots,x_n\}$ be the given $n$ points in the unit cube $Q.$ For each point $x_i,$ there is a set $Y_i$ of $3^d$ points in $\mathbb{R}^d$ such that for any $y \in Y_i, |(x_i)_j-y_j| \in \{0,1\}$ for all $j\in\{1,\cdots,d\}$ where $(x_i)_j$ and $y_j$ denote the $j$-th coordinate of $x_i$ and $y$ respectively. Define $S=Y_1\cup\cdots\cup Y_n$ and write  $S=\{s_1,\cdots,s_m\}.$ It is clear that $m\le 3^dn$. Define
$$ P_{\rho'} = \bigcup_{i=1}^{m} B_{\rho'}(s_i),$$ for some $\rho'>0$, to be specified later. For $A \subset \mathbb{R}^d$, we define the distance $d(x,A)$ between a point $x$ and a set $A$ as
$$ d(x,A):= \inf_{y \in A} \|x-y\|.$$
Let $h_{\rho'}(x)=\inf\left\{ 1, \frac{d(x,P_{\rho'})}{\rho'} \right\}.$ From the construction, $\|h_{\rho'}\|_\infty =1.$ It is straightforward to check that $h_{\rho'}$ is Lipschitz continuous function with its Lipschitz constant less or equal to $1/\rho'.$  Now, consider the normalized indicator function
$$ g_{\rho}(x)=\frac{\mathbbm{1}_{B_{\rho/r}}(x)}{|B_{\rho/r}|},$$
and define
$$ f:=h_{\rho'}*\underbrace{g_{\rho}*\cdots*g_{\rho}}_{r- \text{fold}}=h_{\rho'}*_{r}g_{\rho}.$$
where $|B_{\rho/r}|$ is the volume of ball $B_{\rho/r}$ with center origin and radius $\rho/r$ for some $\rho>0$, which is also specified later. Since $h_{\rho'}$ is Lipschitz continuous function, one can check that $f\in C^r$ using Lemma \ref{convolution} and \cite[Theorem 1-(5)]{hinrichs2014curse}. Note that the support of the $r$-fold convolution of the function $g_{\rho}$ is the $r-$fold Minkowski sum of the balls $B_{\rho/r},$ hence it is $B_{\rho}.$ For simplicity, let us write $g$ as the $r$-fold convolution of the function $g_{\rho}$. 

Assume $\rho' > 2\rho$ for a moment and choose  $x \in B_{\rho}(s_i).$  Recall $f(x)=\int_{\mathbb{R}^d} h_{\rho'}(x-t)g(t)dt.$ If $\|t\| > \rho,$ then since the support of $g$ is $B_{\rho},$ we have $g(t)=0,$ hence $h_{\rho'}(x-t)g(t)=0.$ If $\|t\| \le \rho,$ we have
$$ \|(x-t)-s_i\| \le \|x-s_i\| + \|t\| \le \rho + \rho < \rho'.$$
This implies $x-t \in B_{\rho'}(s_i),$ and therefore $x-t \in P_{\rho'},$ which makes $d(x-t,P_{\rho'})=0.$ From the definition of $h_{\rho'},$ we get $h_{\rho'}(x-t)=0.$ Hence, we always obtain $h_{\rho'}(x-t)g(t)=0,$ which implies $f |_{B_{\rho}(s_i)}=0$. This hold for any $s_i,$ and therefore $f$ vanishes in $\bigcup_{i=1}^{m} B_\rho(s_i).$ We will later choose $\rho'$ and $\rho$ based on this observation.

Now, observe the integration of $f.$ It is obvious that $f \ge 0.$ From the definition of $h_{\rho'},$ we have $h_{\rho'}(x)=1$ if $d(x,P_{\rho'})\ge \rho'.$ Also note that $\|g_\rho\|_1=\int_{\mathbb{R}^d} g_{\rho}(x)dx=1,$ and thus we get $\int_{\mathbb{R}^d}g(x)dx=1.$ Since the support of $g$ is $B_\rho,$ we have $\int_{B_\rho}g(x)dx=1.$ Note that for any $x\notin \bigcup_{i=1}^{m}B_{2\rho'+\rho}(s_i)$ and $z\in B_{\rho'}(s_i),$
\begin{align*}
    \|(x-t)-z\|&\ge \|x-z\|-\|t\| \ge \|x-s_i\|-\|z-s_i\|-\rho \\
    &\ge (2\rho'+\rho)-\rho'-\rho=\rho'
\end{align*}
holds for any $t\in B_\rho.$ Therefore, if $x\notin \bigcup_{i=1}^{m}B_{2\rho'+\rho}(s_i)$, then $d(x,P_{\rho'})\ge \rho'$ and $h_{\rho'}(x-t)=1$ for any $t\in B_\rho.$ Hence, we obtain $$f(x)=\int_{\mathbb{R}^d}h_{\rho'}(x-t)g(t)dt=\int_{B_\rho}h_{\rho'}(x-t)g(t)dt=\int_{B_\rho}1\times g(t)dt=1$$ for  $x\notin \bigcup_{i=1}^{m}B_{2\rho'+\rho}(s_i)$. Using this, the integration of $f$ in $Q$ can be bounded as
\begin{align}\label{integration lowebound}
    \int_{Q}f(x)dx &\ge 1\times |Q\backslash \bigcup_{i=1}^{m}B_{2\rho'+\rho}(s_i)| \ge 1-|\bigcup_{i=1}^{m}B_{2\rho'+\rho}(s_i)| \nonumber\\
    & \ge 1-m\times |B_{2\rho'+\rho}| \nonumber \\
    & \ge 1-n\times 3^d \times \omega_d (2\rho'+\rho)^d
\end{align}
where $\omega_d$ is the volume of a unit ball in $\mathbb{R}^d.$

Next, we check the $C^r$ norm of $f.$ Let $e_j=(0,\cdots,0,1,0,\cdots,0)$ be the $j$th unit vector in $\mathbb{R}^d,$ and denote the volume of a bounded Lebesgue measurable set $X\in\mathbb{R}^d$ as $\text{vol}_d (X)$. Using the same argument in \cite[Section 3]{hinrichs2014curse} and \cite[Section 2]{hinrichs2017product}, for a Lipschitz continuous and bounded function $h(x)$ in $\mathbb{R}^d,$ we get
\allowdisplaybreaks
\begin{align*}
    |D^{e_j}[h*g_{\rho}](x)|&=\frac{1}{\text{vol}_{d}(B_{\rho/r})} \left|\int_{B_{\rho/r}\cap e^{\perp}_{j}}[h(x+s+h_{\text{max}}(s)e_j)-h(x+s-h_{\text{max}(s)}e_j)]ds\right| \nonumber \\
    &\le \frac{2 \text{vol}_{d-1}(B_{\rho/r}\cap e^{\perp}_{j})}{\text{vol}_{d}(B_{\rho/r})} \|h\|_{\infty} \le \frac{2 \text{vol}_{d-1}(B_{\rho/r})}{\text{vol}_{d}(B_{\rho/r})} \|h\|_{\infty} \nonumber\\
    &= \frac{2\omega_{d-1}(\rho/r)^{d-1}}{\omega_{d}(\rho/r)^{d}} \|h\|_{\infty} = \frac{2\omega_{d-1}r}{\omega_{d}\rho}\|h\|_{\infty} \nonumber\\
    &=\frac{k_d r}{\rho}\|h\|_{\infty},
\end{align*}
where
\begin{enumerate}
    \item $e^{\perp}_j$ is the hyperplane orthogonal to $e_j,$
    \item $h_{\text{max}}(s)=\max\{h\ge0 : s+he_j \in B_{\rho/r}\},$
    \item $k_d=\frac{2\omega_{d-1}}{\omega_{d}}>0$ constant depending only on the dimension $d.$ 
\end{enumerate}
This means that we can bound the sup-norm of a derivative as
\begin{align}\label{derivativebound}
    \|D^{e_j} [h*g_\rho]\|_\infty \le \frac{k_dr}{\rho}\|h\|_\infty.
\end{align}
Choose any $\beta \in \mathbb{N}_0^d$ with $|\beta|_1=l\le r.$ Using (\ref{derivativebound}) recursively, we get
\begin{align*}
\|D^{\beta}f\|_\infty&=\|D^\beta (h_{\rho'}*\underbrace{g_{\rho}*\cdots*g_{\rho}}_{r- \text{fold}}) \|_\infty \le (\frac{k_dr}{\rho})^l \|h_{\rho'}*\underbrace{g_{\rho}*\cdots*g_{\rho}}_{(r-l)-\text{fold}}\|_\infty  \\
&\le (\frac{k_dr}{\rho})^l \|h_{\rho'}\|_\infty \times(\|g_\rho\|_1)^{r-l} =(\frac{k_dr}{\rho})^l
\end{align*}
where the last inequality used Young's inequality $r-l$ times. Hence, we have $$\max_{|\beta|_1=l} \|D^\beta f(x)\|_\infty \le (\frac{k_dr}{\rho})^l. $$ If we choose $\rho\le k_dr,$ then we get a $C^r$ norm bound
\begin{align}\label{C^r norm bound}
    \|f\|_{C^r} = \max_{|\beta|_1\le r} \|D^\beta f(x)\|_\infty \le (\frac{k_dr}{\rho})^r.
\end{align}

The final step is normalizing and choosing appropriate $\epsilon_n.$ Set $\rho'=3\epsilon_n, \rho=\epsilon_n,$ and $$F(x)=\frac{f(x)}{(\frac{k_d r}{\epsilon_n})^r}.$$ From the analysis above, $F$ is $C^r(\mathbb{R}^d)$ and $\|F\|_{C^r} \le 1$ if we set $\epsilon_n\le k_dr. $ If we integrate $F$ in the unit cube $Q,$ we get
\allowdisplaybreaks
\begin{align*}
    \int_Q F(x)dx &= (\frac{\epsilon_n}{k_d r})^r \int_Q f(x)dx \\
    & \ge (\frac{\epsilon_n}{k_d r})^r \left(1-n\times 3^d \times \omega_d (2\rho'+\rho)^d\right) \\
    & = (\frac{\epsilon_n}{k_d r})^r \left(1-n\times 3^d \times \omega_d(7\epsilon_n)^d\right) \\
    & = (\frac{\epsilon_n}{k_d r})^r \left(1-n\times \omega_d(21\epsilon_n)^d\right)
\end{align*}
where the inequality used is (\ref{integration lowebound}). Define
\begin{align}\label{Tau}
    \tau:=\min\big\{\frac{1}{21}\big(\frac{1}{2w_d}\big)^{1/d},k_d\big\}.
\end{align}
Now, choose $\epsilon_n = \theta n^{-1/d},$ where $\theta \in (0,\tau]$ is a positive constant depending only on $d.$ Then, it is easy to check 
\begin{enumerate}
    \item $\rho=\epsilon_n\le \tau n^{-1/d} \le \tau\le k_d \le k_dr$,
    \item $ 1-n\times \omega_d(21\epsilon_n)^d = 1-\omega_d\times(21\tau)^d \ge 1/2. $
\end{enumerate}
Then we obtain 
\begin{align*}
    \int_Q F(x)dx\ge \frac{1}{2}\times(\frac{\epsilon_n}{k_dr})^r=\frac{\theta^r}{2k_d^rr^r}n^{-r/d}, \quad \|F\|_{C^r}=(\frac{\epsilon_n}{k_dr})^r\|f\|_{C^r}\le1.
\end{align*}

Let $\psi$ be the restriction of $F$ onto the unit cube $Q$. Denote $K_{\theta,d,r}= \frac{\theta^r}{2k_d^rr^r}$. Then $\|\psi\|_{C^r}\le1$ and $\int_Q \psi(x)dx=\int_Q F(x)dx \ge K_{\theta,d,r}n^{-r/d}$ are obvious. From the previous analysis, we showed that $f$ vanishes in $\bigcup_{i=1}^{m}B_{\epsilon_n}(s_i)$. Therefore, $F$ also vanishes in $\bigcup_{i=1}^{m}B_{\epsilon_n}(s_i)$. By recalling the definition of $s_i$'s, it easy to check $F(x)=0$ when $x \in \bigcup_{i=1}^{n} B'_{\epsilon_n}(x_i).$ Hence, the restriction $\psi$ also satisfies 
$\psi|_{B'_{\epsilon_n}(x_i)}=0$ and therefore $ \fint_{B'_{\epsilon_n}(x_i)} \psi(x)dx = 0$ holds. Finally, we get
\allowdisplaybreaks
\begin{align*}
    &\sup_{\|g\|_{C^r}\le 1} \big\{\frac{1}{n} \sum_{i=1}^{n} \fint_{B'_{\epsilon_n}(x_i)} g dx-\int_{Q} g dx \big\} \\&\ge\frac{1}{n} \sum_{i=1}^{n} \fint_{B'_{\epsilon_n}(x_i)} (-\psi) dx-\int_{Q} (-\psi) dx = \int_{Q} \psi dx\\& \ge K_{\theta,d,r} n^{-r/d}.
\end{align*}
\end{proof}

\section{Proof of Theorem \ref{local}}\label{localproof}

Unlike Lipschitz continuous activation functions, the Rademacher complexity of the unit ball in the Barron space cannot be controlled in the same manner as in~\cite[Lemma A.10]{wojtowytsch2021kolmogorov}, where the Contraction Lemma~\cite[Lemma 26.9]{shalev2014understanding} plays a crucial role in the proof. Fortunately, for a locally Lipschitz continuous activation function $\sigma$ whose Lipschitz constant in $[-x, x]$ is bounded by $O(x^\delta)$, a similar approach to the proof of Theorem~\ref{global} can be employed.

However, the analysis does not extend to infinite-width neural network training for two primary reasons. First, unlike Lemma~\ref{Control of Barron norm}, probability measures with finite second moments do not necessarily guarantee a well-defined integral representation. For instance, consider the case where $d=1$, $\sigma(x) = \max\{0, x\}^2$, and the probability measure $\pi$ is defined as $\pi(a=n, b=0, w=n) = \frac{1}{K n^4}, \quad \text{for all } n \in \mathbb{N}$, where $K = \sum_{i=1}^{\infty} \frac{1}{i^4} < \infty$. It is straightforward to verify that $N(\pi) < \infty$, yet the integral representation~\eqref{intrep} is undefined. While one might consider other notions of Barron spaces~\cite{heeringa2024embeddings,siegel2022high,siegel2020approximation,siegel2024sharp}, such settings make the control of Barron norms via $N(\pi^t)$-as in Lemma \ref{Control of Barron norm}-intractable. Consequently, we cannot track the Barron norm of infinite-width neural networks under gradient flow, precluding the application of the same analysis.

Second, probability measures that define infinite-width neural networks do not impose uniform bounds on the parameters, making it difficult to control the Rademacher complexity. Nevertheless, when focusing on finite-width neural networks, the setting relevant to practical implementations, the curse of dimensionality can still be established. Since the set of shallow networks with $m$ neurons does not constitute a normed vector space, the proof relies on Lemma~\ref{poor1} instead of Lemma~\ref{poor2}.

In this section, we consider probability distributions of the form
\begin{align}\label{finitedist}
    \frac{1}{m}\sum_{i=1}^{m}\delta_{(a_i,w_i,b_i)}
\end{align}
which correspond to a two-layer neural network $f(x)=\frac{1}{m}\sum_{i=1}^{m}a_i\sigma(w_i^Tx+b_i)$ with $m$ neurons in the mean-field setting. All distribution $\pi_m$ discussed in this section adhere to the form given in (\ref{finitedist}). Note that if $\pi_m^t$ evolves by the 2-Wasserstein gradient flow of the risk function (\ref{risk}), then $\pi_m^t$ remains in the form of (\ref{finitedist}) at all times. This follows from the the observations in~\cite[Proposition B.1]{chizat2018global} and~\cite[Lemma 3]{wojtowytsch2020can}.

We begin with some definitions. Define two sets $F_{m,D}$ and $S_D$ as
\begin{align*}
    &F_{m,D}:=\{f_{\pi_m} : \pi_{m}=\frac{1}{m}\sum_{i=1}^{m}\delta_{(a_i,w_i,b_i)}, N(\pi_m)\le D\},\\
    &S_D:=\{a\sigma(w^Tx+b) : a^2+\|w\|_2^2+b^2 \le D\}.
\end{align*}
From the definition, it is obvious that any element in $F_{m,D}$ can be expressed as a convex combination of single neurons of the form $a\sigma(w^Tx+b)$ in $S_{mD}$. Hence, $F_{m,D}$ is a subset of the convex hull of $S_{mD}$, which implies for any finite set $S \in \mathbb{R}^d,$
\begin{align*}
    \mathrm{Rad}(F_{m,D},S) \le \mathrm{Rad}(\text{conv}(S_{mD}),S).
\end{align*}
holds, where $\mathrm{Rad}(F,S)$ denotes Rademacher complexity of $F$ with respect to $S$. For further details on Rademacher complexities, see~\cite[Chapter 26]{shalev2014understanding}.

Let $\sigma$ be a locally Lipschitz continuous function and denote $L_k=\sup_{x\neq y, x,y\in [-k,k]} \frac{|\sigma(x)-\sigma(y)|}{|x-y|}$ the Lipschitz constant of $\sigma$ in the domain $[-k,k].$
\begin{lemma}\label{lipschitz}
    If $\|w\|^2+b^2 \le mD,$ then $|\sigma(w^Tx+b)-\sigma(w^Ty+b)|\le L_{\sqrt{mD(d+1)}}\|w^T(x-y)\|$ holds for any $x,y\in Q.$
\end{lemma}
\begin{proof}
    For any $x \in Q, |w^Tx+b|\le \sqrt{(\|w\|_2^2+b^2)(d+1)}\le \sqrt{mD(d+1)}$ holds. Then the inequality holds due to the definition of Lipschitz constant.
\end{proof}

Now we give an estimate on the empirical Rademacher complexity of $S_{mD}.$
\begin{lemma}\label{locallipchitzrad}
    Let $\{X_1,\cdots,X_n\} \subset Q$. Then we have
    \begin{align}
        \mathrm{Rad}(S_{mD},\{X_1,\cdots,X_n\}) \le \frac{L_{\sqrt{mD(d+1)}}\times mD\sqrt{d+1}}{2\sqrt{n}}.
    \end{align}
\end{lemma}
\begin{proof}
    Using Lemma \ref{lipschitz} and the Contraction Lemma \cite[Lemma 26.9]{shalev2014understanding}, we get
    \allowdisplaybreaks
    \begin{align*}
        \mathrm{Rad}(S_{mD},&\{X_1,\cdots,X_n\}) =\mathbb{E}_{\zeta}\Big[\sup_{(a,b,c):a^2+\|w\|_2^2+b^2 \le mD} \frac{1}{n}\sum_{i=1}^{n} \zeta_i a\sigma(w^T X_i+b) \Big] \\
        &\le L_{\sqrt{mD(d+1)}}\times \mathbb{E}_{\zeta}\Big[\sup_{(a,b,c):a^2+\|w\|_2^2+b^2 \le mD} \frac{1}{n}\sum_{i=1}^{n} \zeta_i a(w^T X_i+b)\Big] \\
        &=\frac{L_{\sqrt{mD(d+1)}}}{n}\times\mathbb{E}_{\zeta}\Big[ \sup_{(a,b,c):a^2+\|w\|_2^2+b^2 \le mD} \sum_{i=1}^{n} \zeta_i \langle a(w^T,b)^T, (X_i^T,1)^T \rangle\Big] \\
        &= \frac{L_{\sqrt{mD(d+1)}}}{n}\times\mathbb{E}_{\zeta}\Big[ \sup_{(a,b,c):a^2+\|w\|_2^2+b^2 \le mD}  \langle a(w^T,b)^T, \sum_{i=1}^{n} \zeta_i(X_i^T,1)^T \rangle\Big] \\
        &\le \frac{L_{\sqrt{mD(d+1)}}}{n}\times\mathbb{E}_{\zeta}\Big[ \sup_{(a,b,c):a^2+\|w\|_2^2+b^2 \le mD}  \|a(w^T,b)^T\|_2\times\|\sum_{i=1}^{n} \zeta_i(X_i^T,1)^T \|_2\Big]\\
        &\le \frac{L_{\sqrt{mD(d+1)}}}{n}\times\mathbb{E}_{\zeta}\Big[ \frac{mD}{2}\times\|\sum_{i=1}^{n} \zeta_i(X_i^T,1)^T \|_2\Big]\\ 
        &\le \frac{L_{\sqrt{mD(d+1)}}mD}{2n}\times \Big(\mathbb{E}_\zeta \big[|\sum_{i=1}^{n} \zeta_i(X_i^T,1)^T \|_2]^2 \big]\Big)^{1/2} \\
        &\le \frac{L_{\sqrt{mD(d+1)}}mD}{2n}\times \big( n\times \max_{i} \|(X_i^T,1)^T\|_2^2 \big)^{1/2} \\
        &\le \frac{L_{\sqrt{mD(d+1)}}mD}{2n}\times \sqrt{n\times (d+1)} \\
        &=\frac{L_{\sqrt{mD(d+1)}}\times mD\sqrt{d+1}}{2\sqrt{n}}.
    \end{align*}
\end{proof}

\begin{corollary}\label{Loc_rademacher_control}
    Let $\{X_1,\cdots,X_n\} \subset Q$. Then, $$\mathrm{Rad}(F_{m,D},\{X_1,\cdots,X_n\}) \le \frac{L_{\sqrt{mD(d+1)}}\times mD\sqrt{d+1}}{2\sqrt{n}}.$$
\end{corollary}
\begin{proof}
    Note that
    $$\mathrm{Rad}(F_{m,D},\{X_1,\cdots,X_n\}) \le \mathrm{Rad}(\text{conv}(S_{mD}),\{X_1,\cdots,X_n\}) =\mathrm{Rad}(S_{mD},\{X_1,\cdots,X_n\}).$$ Now use Lemma \ref{locallipchitzrad}.
\end{proof}

Finally, we introduce two lemmas that assist our proof. Let $U_Q$ be the uniform distribution on $Q.$
\begin{lemma}\label{L2 estimate}
    Let $Z=L^2(Q).$ Fix $0<\gamma\ll1$ and set $\epsilon_n=\gamma n^{-1/d}$ for $n\in\mathbb{N}$. Then,
    \begin{align*}
        \mathbb{E}_{X_i \overset{\mathrm{iid}}{\sim} U_Q}\sup_{\phi\in B^Z}\Big[\frac{1}{n}\sum_{i=1}^{n}\fint_{B'_{\epsilon_n}(X_i)}\phi(x)dx-\int_Q
    \phi(x)dx\Big] \le \sqrt{\frac{1+a_d\gamma^d}{b_d\gamma^d}}
    \end{align*}
    holds where $a_d$ and $b_d$ are positive constant depending only on $d$.
\end{lemma}
The proof of Lemma \ref{L2 estimate} can be found in the proof of \cite[Lemma 3.3]{wojtowytsch2021kolmogorov}.
\begin{lemma}\label{ball integral control}
    Let $F$ be a a subset of $C(Q)$.  Then for any $0<\epsilon<1,$ we have
    \begin{align*}
        \mathbb{E}_{X_i \overset{\mathrm{iid}}{\sim} U_Q}\sup_{f\in F}\big[\frac{1}{n}\sum_{i=1}^{n}\fint_{B'_{\epsilon}(X_i)}f(x)dx-\int_Qf(x)dx] \le \mathbb{E}_{X_i \overset{\mathrm{iid}}{\sim} U_Q}\sup_{f\in F}\big[\frac{1}{n}\sum_{i=1}^{n}f(X_i)-\int_Qf(x)dx].
    \end{align*}
\end{lemma}
\begin{proof}
    In this proof, we interpret $X_i+z$ as a shift on the flat torus. Note that for a fixed $z$, $X_i$ and $X_i+z$ have the same distribution. Therefore we have
    \allowdisplaybreaks
    \begin{align*}
        &\mathbb{E}_{X_i \overset{\mathrm{iid}}{\sim} U_Q}\sup_{f\in F}\Big[\frac{1}{n}\sum_{i=1}^{n}\fint_{B'_{\epsilon}(X_i)}f(x)dx-\int_Qf(x)dx\Big] \\
        =&\mathbb{E}_{X_i \overset{\mathrm{iid}}{\sim} U_Q}\sup_{f\in F}\Big[\frac{1}{n}\sum_{i=1}^{n}\fint_{B'_{\epsilon}(0)}f(X_i+z)dz-\int_Qf(x)dx\Big] \\
        =&\mathbb{E}_{X_i \overset{\mathrm{iid}}{\sim} U_Q}\sup_{f\in F}\Big[\fint_{B'_{\epsilon}(0)}\frac{1}{n}\sum_{i=1}^{n}\big(f(X_i+z)-\int_Qf(x)dx\big)dz\Big]\\
        \le &\mathbb{E}_{X_i \overset{\mathrm{iid}}{\sim} U_Q}\fint_{B'_{\epsilon}(0)}\sup_{f\in F}\Big[\frac{1}{n}\sum_{i=1}^{n}\big(f(X_i+z)-\int_Qf(x)dx\big)\Big]dz \\
        = &\fint_{B'_{\epsilon}(0)}\mathbb{E}_{X_i \overset{\mathrm{iid}}{\sim} U_Q}\sup_{f\in F}\Big[\frac{1}{n}\sum_{i=1}^{n}\big(f(X_i+z)-\int_Qf(x)dx\big)\Big]dz\\
        = &\mathbb{E}_{X_i \overset{\mathrm{iid}}{\sim} U_Q}\sup_{f\in F}\Big[\frac{1}{n}\sum_{i=1}^{n}f(X_i)-\int_Qf(x)dx\Big].
    \end{align*}
\end{proof}

To prove Theorem \ref{local}, we begin with a lemma which is similar to Lemma \ref{poor} for locally Lipschitz activation functions that satisfy $L_t=O(t^\delta)$. We prove that for fixed $m\in \mathbb{N}$, there exists $\phi\in C^r(Q)$ which is poorly approximated by shallow neural networks with width $m$.
\begin{lemma}\label{loclipclaim}
    Let $\sigma:\mathbb{R}\rightarrow\mathbb{R}$ be a locally Lipschitz function with $L_t=O(t^\delta)$ for some $\delta\ge 0.$ Fix $m\in \mathbb{N}.$ Assume $r<\frac{d}{2}.$ Then there exists $\phi \in C^r(Q)$ such that
    \begin{align*}
        \limsup_{t\rightarrow\infty}\big(t^\gamma \inf_{f\in F_{m,t}}\|\phi-f\|_{L^2(Q)}\big)=\infty.
    \end{align*}
    holds for every $\gamma>\frac{(2+\delta)r}{d-2r}.$
\end{lemma}

\begin{proof}
Since the approximators have the same number of neurons, their union does not form a vector space. Therefore, instead of applying Lemma \ref{poor2}, we utilize Lemma \ref{poor1}.

Let $Y=C^r(Q)$ equipped with the norm defined in Lemma \ref{numerical_cod}, $Z=L^2(Q)$, and $W=\mathbb{R}$. Given that $L_t=O(t^\delta)$, there exist $T=T_\sigma>0$ and $C=C_\sigma>0$ such that $L_t\le Ct^\delta$ for $t\ge T$. Define $A\in L(Z,W)$ be defined as in (\ref{operators}) in the proof of Theorem \ref{poor}. For any $n\in\mathbb{N}$, choose $\epsilon_n=\gamma n^{-1/d}$, as done in the proof of Theorem \ref{poor}. 

Note that for any super-exponentially increasing sequence $\{n_k\}_{k\ge1}$ and $m_k=n_k^{[\sqrt{k}]}$, we have
\begin{align}\label{inftylimit}
\lim_{k\rightarrow\infty}\frac{m_k^{\frac{1}{2}-\frac{r}{d}}}{n_k}=\lim_{k\rightarrow\infty}n_k^{(\frac{1}{2}-\frac{r}{d})[\sqrt{k}]-1}=\infty.
\end{align}
Therefore, there exists $k_0\in\mathbb{N}$ such that
\begin{equation}\label{inequality}
T^2\le \Big(\frac{m_k^{\frac{1}{2}-\frac{r}{d}}K_{\gamma,d,r}}{24n_kC m^{1+\frac{\delta}{2}}(d+1)^{\frac{1}{2}+\frac{\delta}{2}}}\Big)^{\frac{1}{1+\frac{\delta}{2}}} 
\end{equation}
for all $k\ge k_0,$ where $K_{\gamma,d,r}$ is the constant in Lemma \ref{numerical_cod}. Now, define $n'_k=n_{k+k_0}$. Then $n'_k$ is also a super-exponentially increasing sequence because it is strictly increasing with $n'_1=n_{k_0+1}>n_1\ge2$ and
\begin{align*}
    \sum_{l>k}\frac{1}{n'_l}=\sum_{l>k+k_0}\frac{1}{n_l}\le\frac{2}{n_{k+k_0}^{k+k_0+1}}\le\frac{2}{n_{k+k_0}^{k+1}}=\frac{2}{(n'_k)^{k+1}}
\end{align*}
holds for all $k\in\mathbb{N}.$ Define $m'_k=(n'_k)^{[\sqrt{k}]}.$ Then for any $k\ge1,$
\begin{equation*}
T^2\le \Big(\frac{(m_k')^{\frac{1}{2}-\frac{r}{d}}K_{\gamma,d,r}}{24n_k'C m^{1+\frac{\delta}{2}}(d+1)^{\frac{1}{2}+\frac{\delta}{2}}}\Big)^{\frac{1}{1+\frac{\delta}{2}}} 
\end{equation*}
holds. From this observation, we can choose a super-exponentially increasing sequence $\{n_k\}_{k\ge1}$ and $m_k=n_k^{[\sqrt{k}]}$ such that inequality (\ref{inequality}) holds for all $k\ge1.$

Choose a super-exponentially increasing sequence $\{n_k\}_{k\ge1}$ and $m_k=n_k^{[\sqrt{k}]}$ that satisfy (\ref{inequality}) for all $k\ge1$. Define time scales $\{t_k\}_{k\ge1}$ as $t_k=\Big(\frac{m_k^{\frac{1}{2}-\frac{r}{d}}K_{\gamma,d,r}}{24n_kC m^{1+\frac{\delta}{2}}(d+1)^{\frac{1}{2}+\frac{\delta}{2}}}\Big)^{\frac{1}{1+\frac{\delta}{2}}}.$ From (\ref{inequality}), $t_k\ge T^2$ holds for all $k\ge1.$ Now, let us find appropriate $\{A_n\}_{n\ge1} \in L(Z,W)$ for Lemma \ref{poor1}.

When $n\not\in\{m_k\}_{k\ge1}$, define $A_n\in L(Z,W)$ as in the proof of Theorem \ref{poor}. Then we have
\begin{align*}
    \|A_n-A\|_{L(Y,W)} \ge K_{\gamma,d,r} n^{-r/d}, \quad \|A_n-A\|_{L(Z,W)} \le 3C_{d,\gamma}
\end{align*}
for all $n\not\in\{m_k\}_{k\ge1}.$ The constants $K_{\gamma,d,r}$ and $C_{d,\gamma}$ are are idential to those used in the proof of Theorem \ref{poor}.

When $n=m_k,$ we construct the linear operators $\{A_{m_k}\}_{k\ge1}$ as follows. Since $t_k\ge T^2$ for all $k\ge1$, $\sqrt{mt_k(d+1)}\ge T$ hold, which implies $L_{\sqrt{mt_k(d+1)}}\le C(\sqrt{mt_k(d+1)})^\delta$. From Lemma \ref{ball integral control} and Corollary \ref{Loc_rademacher_control},
\allowdisplaybreaks
\begin{align}\label{expectationbound}
    &\mathbb{E}_{X_i \overset{\mathrm{iid}}{\sim} U_Q}\sup_{f\in F_{m,t_k}}\big[\frac{1}{m_k}\sum_{i=1}^{m_k}\fint_{B'_{\epsilon_{m_k}}(X_i)}f(x)dx-\int_Qf(x)dx] \nonumber \\
    \le &\mathbb{E}_{X_i \overset{\mathrm{iid}}{\sim} U_Q}\sup_{f\in F_{m,t_k}}\big[\frac{1}{m_k}\sum_{i=1}^{m_k}f(X_i)-\int_Qf(x)dx] \nonumber \\
    \le &2\times\mathbb{E}_{X_i \overset{\mathrm{iid}}{\sim} U_Q} \mathrm{Rad}(F_{m,t_k},\{X_1,\cdots,X_{m_k}\}) \nonumber \\
    \le& 2\times\frac{L_{\sqrt{mt_k(d+1)}}\times mt_k\sqrt{d+1}}{2\sqrt{m_k}} \nonumber \\
    \le &\frac{C(\sqrt{mt_k(d+1)})^\delta mt_k\sqrt{d+1}}{\sqrt{m_k}} \nonumber \\
    =&C m^{1+\frac{\delta}{2}}(d+1)^{\frac{1}{2}+\frac{\delta}{2}}\frac{t_k^{1+\frac{\delta}{2}}}{\sqrt{m_k}}=\frac{K_{\gamma,d,r}m_k^{-\frac{r}{d}}}{24n_k}
\end{align}
holds. The second inequality holds because the expected value of the representativeness is bounded by twice the expected Rademacher complexity~\cite[Lemma 26.2]{shalev2014understanding}. Now with Lemma \ref{L2 estimate} and (\ref{expectationbound}), by using a simple probabilistic argument using the Markov inequality, there exist $m_k$ points $\{X^1_{m_k},\cdots,X^{m_k}_{m_k}\} \subset Q$ such that
\begin{align*}
        &\sup_{\phi \in B^Y} \Big[ \frac{1}{m_k} \sum_{i=1}^{m_k} \fint_{B'_{\epsilon_{m_k}}(X^i_{m_k})} \phi dx-\int_{Q} \phi dx \Big] \ge K_{\gamma,d,r} m_k^{-r/d}, \\
        &\sup_{\phi \in B^Z} \Big[ \frac{1}{m_k} \sum_{i=1}^{m_k} \fint_{B'_{\epsilon_{m_k}}(X^i_{m_k})} \phi dx-\int_{Q} \phi dx \Big] \le 3\sqrt{\frac{1+a_d\gamma^d}{b_d\gamma^d}},\\
        &\sup_{f\in F_{m,t_k}} \Big[\frac{1}{m_k}\sum_{i=1}^{m_k}\fint_{B'_{\epsilon_{m_k}}(X^i_{m_k})}f(x)dx-\int_Qf(x)dx \Big]  \le \frac{K_{\gamma,d,r}m_k^{-\frac{r}{d}}}{8n_k}    
\end{align*}
hold, where the first inequality is due to Lemma \ref{numerical_cod}. Now using these points $\{X^1_{m_k},\cdots,X^{m_k}_{m_k}\} \subset Q$, we define $\{A_{m_k}\}_{k\ge1}\subset L(Z,W)$ as
\begin{align*}
    A_{m_k}(\phi)=\frac{1}{m_k} \sum_{i=1}^{m_k} \fint_{B'_{\epsilon_{m_k}}(X^i_{m_k})} \phi dx.
\end{align*}

With the constructed linear operators $\{A_n\}_{n\ge1}$, by setting $X_k=F_{m,t_k}$ for $k\ge1$, one can easily verify that the conditions in Lemma \ref{poor1} are satisfied with $\alpha=\frac{1}{2}, \beta=\frac{r}{d}, c_Y=K_{\gamma,d,r}$ and $C_Z=\max\big\{3\sqrt{\frac{1+a_d\gamma^d}{b_d\gamma^d}},3C_{d,\gamma}\big\}.$ Now from Lemma \ref{poor1}, there exists $\phi\in B^Y$ such that for every $\gamma>\frac{r/d}{1/2-r/d},$ we have
\begin{align}\label{loc_first}
\limsup_{k\rightarrow\infty}\Big[ \big(\frac{m_k^{1/2-r/d}}{n_k}\big)^\gamma \inf_{f\in F_{m,t_k}} \|f-\phi\|_{L^2(Q)} \Big]=\infty.
\end{align}
Note that $\frac{m_k^{\frac{1}{2}-\frac{r}{d}}}{n_k}=At_k^{1+\frac{\delta}{2}}$ where $A=\frac{24C m^{1+\frac{\delta}{2}}(d+1)^{\frac{1}{2}+\frac{\delta}{2}}}{K_{\gamma,d,r}}$ is a constant that does not depend on $k$. Hence, (\ref{loc_first}) can be written as
\begin{align}\label{loc_second}
\limsup_{k\rightarrow\infty}\Big[t_k^{\gamma(1+\frac{\delta}{2})} \inf_{f\in F_{m,t_k}} \|f-\phi\|_{L^2(Q)} \Big]=\infty
\end{align}
for every $\gamma>\frac{r/d}{1/2-r/d}.$ From (\ref{inftylimit}), $\lim_{k\rightarrow\infty}t_k=\infty$ holds, and with (\ref{loc_second}), we conclude
\begin{align*}
\limsup_{t\rightarrow\infty}\big(t^\gamma \inf_{f\in F_{m,t}}\|\phi-f\|_{L^2(Q)}\big) \ge \limsup_{k\rightarrow\infty}\big( t_k^{\gamma} \inf_{f\in F_{m,t_k}} \|f-\phi\|_{L^2(Q)} \big)=\infty.
\end{align*}
for every $\gamma>(1+\frac{\delta}{2})\times\frac{r/d}{1/2-r/d}=\frac{(2+\delta)r}{d-2r}$. This finishes the proof of Lemma \ref{loclipclaim}.
\end{proof}

We now present the proof of Theorem \ref{local}.
\begin{proof}
    Choose any $\phi\in C^r(Q)$ that satisfies Lemma \ref{loclipclaim}. Let $\pi_m^0$ be the initial parameter distribution at time $t=0$, and let $\pi^t_m$ denote the evolution of $\pi_m^0$ under the Wasserstein gradient flow at time $t>0$. By Lemma \ref{Second moment evolution}, there exists a positive constant $K=K_{\pi^0_m,\phi}$ such that $N(\pi^t_m)\le Kt$ holds for all $t\ge1.$ Note that the population risk at time $t$, $R(\pi^t_m),$ is equal to $\frac{1}{2}\|\phi-f_{\pi^t_m}\|^2_{L^2(Q)}.$ Therefore by Lemma \ref{loclipclaim},
    \allowdisplaybreaks
    \begin{align*}
        \limsup_{t\rightarrow\infty}\big[t^\gamma R(\pi^t_m)\big]&=\frac{1}{2}\limsup_{t\rightarrow\infty}\big(t^\gamma \|\phi-f_{\pi^t_m}\|_{L^2(Q)}^2\big) \\
        &\ge\frac{1}{2} \limsup_{t\rightarrow\infty}\big(t^\gamma\inf_{N(\pi_m)\le Kt}\|\phi-f_{\pi_m}\|_{L^2(Q)}^2\big) \\
        &= \frac{1}{2} \limsup_{t\rightarrow\infty}\big(t^\gamma \inf_{f\in F_{m,Kt}}\|\phi-f\|_{L^2(Q)}^2\big)\\
        &= \frac{1}{2}(\frac{1}{K})^\gamma \times\limsup_{t\rightarrow\infty}\big(t^\gamma \inf_{f\in F_{m,t}}\|\phi-f\|_{L^2(Q)}^2\big)=\infty
    \end{align*}
    holds for every $\gamma>2\times\frac{(2+\delta)r}{d-2r}=\frac{(4+2\delta)r}{d-2r}.$
\end{proof}

\end{document}